\documentclass[twoside]{article}

\usepackage[accepted]{aistats2021}

\usepackage{microtype}
\usepackage{graphicx}
\usepackage{booktabs} 
\usepackage{caption}
\usepackage{subcaption}
\usepackage{natbib}
\newcounter{dbaCounter}

\newcounter{ggiCounter}

\newcounter{ybCounter}

\newcounter{wCounter}

\newcounter{imgCounter}

\usepackage{textcomp}

\usepackage{amsmath,amsfonts,bm,amsthm}
\usepackage{algorithm,algorithmic,natbib}

\usepackage{amsmath,amsfonts,bm}
\usepackage{algorithm,algorithmic}
\usepackage{color,xcolor}

\definecolor{orangeblind}{HTML}{E69F00}
\definecolor{blueblind}{HTML}{1A77C9}
\definecolor{greenblind}{HTML}{107D60}
\definecolor{beaublue}{rgb}{0.74, 0.83, 0.9}
\definecolor{babyblueeyes}{rgb}{0.63, 0.79, 0.95}
\definecolor{mydarkblue}{rgb}{0,0.08,0.45}
\usepackage[colorlinks=true,
    linkcolor=mydarkblue,
    citecolor=mydarkblue,
    filecolor=mydarkblue,
    urlcolor=mydarkblue]{hyperref}  
\definecolor{myblue}{HTML}{D3E4FC}

\newcommand{\Val}{\mathrm{V}}

\usepackage{pifont}

\usepackage{todonotes}

\DeclareMathOperator{\relu}{ReLU}

\newtheorem{assumption}{Assumption}

\newtheorem{definition}{Definition}

\newtheorem{example}{Example}

\DeclareMathOperator{\hull}{hull}
\DeclareMathOperator{\closure}{closure}

\def\1{\bm{1}}

\DeclareMathAlphabet{\mathsfit}{\encodingdefault}{\sfdefault}{m}{sl}
\SetMathAlphabet{\mathsfit}{bold}{\encodingdefault}{\sfdefault}{bx}{n}

\def\gF{{\mathcal{F}}}
\def\gG{{\mathcal{G}}}
\def\gH{{\mathcal{H}}}

\def\gP{{\mathcal{P}}}

\def\gX{{\mathcal{X}}}

\def\sN{{\mathbb{N}}}

\def\sP{{\mathbb{P}}}

\def\sR{{\mathbb{R}}}

\newcommand{\E}{\mathbb{E}}

\newcommand{\R}{\mathbb{R}}

\DeclareMathOperator{\diam}{diam}

\newlength\myindent
\usepackage[super]{nth}
\usepackage[inline, shortlabels]{enumitem}  

\usepackage[T1]{fontenc}

\usepackage{natbib}
\usepackage{graphicx}
\usepackage{amsthm,amsmath,amssymb}
\usepackage{todonotes}
\usepackage{color,xcolor}
\usepackage{fdsymbol}
\usepackage{pgfplotstable} 
\usetikzlibrary{automata, positioning, arrows, shapes, fit, calc, intersections, patterns}
\usepgfplotslibrary{statistics}
\usepgfplotslibrary{fillbetween} 

\makeatletter
\newcommand{\newreptheorem}[2]{\newtheorem*{rep@#1}{\rep@title}\newenvironment{rep#1}[1]{\def\rep@title{#2 \ref*{##1}}\begin{rep@#1}}{\end{rep@#1}}}
\makeatother

\newreptheorem{theorem}{Theorem}
\newreptheorem{example}{Example}

\usepackage{thmtools}
\usepackage{thm-restate}

\usetikzlibrary{math}
\declaretheorem[name=Theorem]{theorem}
\declaretheorem[name=Lemma]{lemma}

\begin{document}

%

%
\runningtitle{A Limited-Capacity Minimax Theorem for Nonconvex-Nonconcave Games}
\runningauthor{Gauthier Gidel, David Balduzzi, Wojciech Marian Czarnecki, Marta Garnelo, Yoram Bachrach}
\twocolumn[

\aistatstitle{A Limited-Capacity Minimax Theorem for Non-Convex Games or: \\[1mm] \large How I Learned to Stop Worrying about Mixed-Nash and Love Neural Nets}

\vspace{-4mm}
\aistatsauthor{ 
Gauthier Gidel$^{\dag,\spadesuit}$
\And 
David Balduzzi$^{\spadesuit}$} 
\aistatsaddress{Mila, Universit\'e de Montr\'eal  \And XTX Markets }
\vspace{-2mm}
\aistatsauthor{Wojciech Marian Czarnecki
 \And  
 Marta Garnelo
 \And 
 Yoram Bachrach}
 \aistatsaddress{DeepMind \And DeepMind \And DeepMind}
]

\begin{abstract}
  Adversarial training, a special case of multi-objective optimization, is an increasingly prevalent machine learning technique: some of its most notable applications include GAN-based generative modeling and self-play techniques in reinforcement learning which have been applied to complex games such as Go or Poker. In practice, a \emph{single} pair of networks is typically trained to find an approximate equilibrium of a highly nonconcave-nonconvex adversarial problem. However, while a classic result in game theory states such an equilibrium exists in concave-convex games, there is no analogous guarantee if the payoff is nonconcave-nonconvex. Our main contribution is to provide an approximate minimax theorem for a large class of games where the players pick neural networks including WGAN, StarCraft II and Blotto Game. Our findings rely on the fact that despite being nonconcave-nonconvex with respect to the neural networks parameters, these games are concave-convex with respect to the actual models (e.g., functions or distributions) represented by these neural networks. 
\end{abstract}

\section{INTRODUCTION}

Real-world games have been used as benchmarks in artificial intelligence for decades~\citep{samuel1959some, tesauro1995temporal}, with recent progress on increasingly complex domains such as poker~\citep{brown2017safe, brown2019superhuman}, chess,  Go~\citep{silver2017mastering}, and StarCraft II~\citep{vinyals2019grandmaster}. Simultaneously, remarkable advances in generative modeling of images~\citep{wu2019logan} and speech synthesis~\citep{binkowski2020high} have resulted from two-player games \emph{explicitly} designed to facilitate via carefully constructed arms races~\citep{goodfellow2014generative}. Two-player zero-sum games also play a fundamental role in building classifiers that are robust to adversarial attacks~\citep{madry2017towards}.    

The goal of the paper is to put learning---by neural nets---in two-player zero-sum games on a firm theoretical foundation to answer the question: \emph{What does it mean to solve a game with neural nets?}

In single-objective optimization, performance is well-defined via a fixed objective. However, it is not obvious what counts as ``optimal" in a two-player zero-sum nonconcave-nonconvex setting. 
Since each player's goal is to maximize its payoff, it is natural to ask whether a player can maximize its utility independently of how the other player behaves. \citet{von1944theory} laid the foundation of game theory with the Minimax theorem, which provides a meaningful notion of optimal behavior against an unknown adversary. 
For a two-player zero-sum game, such a solution concept incorporates two notions:
\begin{enumerate*}[itemjoin = \;\,, label=(\roman*)]
        \item \emph{a value} $\Val$,  
    \item \emph{a strategy for each player} such that their average gain is at least $\Val$ (resp. -$\Val$) no matter what the other does.
\end{enumerate*}
The existence of such a value and optimal strategies in concave-convex games is guaranteed in~\citet{sion1958general}, an extension of von Neumann's result.

From a game-theoretic perspective, minimax may not exist in nonconcave-nonconvex (NC-NC) games. Nevertheless, machine learning (ML) practitioners typically train a \emph{single} pair of neural networks to solve 
\begin{equation} \label{eq:nonconcave-nonconvex}
    \min_{\theta \in  \Theta} \max_{w \in \Omega} \varphi(w,\theta) 
    \quad \text{where} \quad \varphi(\cdot,\cdot) \; \text{is NC-NC} \,.
\end{equation}
Previous work~\citep{arora2017generalization,hsieh2019finding,domingo2020mean} coped with this nonconcave-nonconvexity issue by transforming Eq.~\ref{eq:nonconcave-nonconvex} into a bilinear minimax problem over the Borel distributions on $\Theta$ and $\Omega$ (a.k.a. lifting trick), 
\begin{equation}\label{eq:lifted_minimax}
    \min_{\mu \in  \gP(\Theta)} \max_{\nu \in \gP(\Omega)} \langle \mu, A\nu \rangle :=
        \E_{\substack{\theta \sim \mu\\ w \sim \nu}}[\varphi(w,\theta) ]
\end{equation}
However, working on the space of distributions (a.k.a, mixed strategies) over the weights of a neural network is not practical and does not exactly correspond to the initial problem~\eqref{eq:nonconcave-nonconvex}. That is why we do not consider~\eqref{eq:lifted_minimax} and focus instead on minimax theorems for~\eqref{eq:nonconcave-nonconvex}. 

Our main contribution is Theorem~\ref{thm:non-conv_minimax}, an approximate limited-capacity minimax theorem for NC-NC games. This result contrasts with the negative result of~\citet{jin2019local} who construct a NC-NC game where pure global minimax does not exist. The insights provided by our main theorem are three-fold; first, it provides a principled explanation of why practitioners have successfully trained a single pair of neural nets in games like GANs. Secondly, the equilibrium achieved in the theorem has a meaningful interpretation as the solution of a game where the players have \emph{limited-capacity}. Finally, it shed light on some geometric properties of the space of neural networks (roughly, it is 'almost convex'), that could be leveraged in diverse contexts of deep learning theory.
All the proofs can be found in the appendix.


\section{RELATED WORK}

\paragraph{Minimax theorems in GANs.}
Many papers have adopted a game-theoretic perspective on GANs, motivating the computation of distributions of generators (in practice, finite collections)~\citep{arora2017generalization,oliehoek2018beyond,hsieh2019finding,grnarova2017online,domingo2020mean} or an average of discriminators~\citet{Farnia2018convex}.
However, these papers fail to explain why, in practice, it suffices to train only a single generator and discriminator (instead of collections) to achieve state-of-the-art performance~\citep{brock2019large}.
Overall, even if we share motivations with the related work mentioned above (providing principled results), our results and conclusion are fundamentally different: we explain why using a single generator and discriminator---not a distribution over them---makes sense. We do so by proving that one can achieve a notion of nonconcave-nonconvex minimax equilibrium in GANs parametrized with neural nets.

\paragraph{Stackelberg games and local optimality.}
The literature has considered other notions of equilibrium.
Recently,~\citet{fiez2020convergence} proved results on games where the goal is to find a (local) Stackelberg equilibrium. Using that perspective,~\citet{zhang2020optimality} and \citet{jin2019local} studied local-optimality in the context of nonconcave-nonconvex games.
Alternative notions of local equilibria have been proposed and studied by~\citet{berard2020closer, farnia2020gans, schaefer2020implicit} who argued that local Nash equilibria may not be the right notion of local optimality.
Our work provides a complementary perspective by providing a \emph{global} minimax optimality theorem in a restricted---though realistic---nonconcave-nonconvex setting.
Stackelberg equilibria may be meaningful in some contexts, such as adversarial training, but we argue in \S\ref{app:importance of minimax} that the minimax theorem is fundamental for defining a valid notion of solution for many machine learning applications.

\paragraph{Parametrized strategies in games.}
The notion of latent matrix games mentioned in this paper is similar to the pushforward technique proposed by~\citet{dou2019finding}. It can also be related to the latent policies used in some multi-agent reinforcement learning (RL) applications.
For instance the agent trained by~\citet{vinyals2019grandmaster} to play the game of StarCraft II had policies of form $\pi(a|s,z)$ where $z$ belongs to structured space that corresponds to a particular way to start the game or to actions it should complete during the game (e.g., first constructed buildings). 
Moreover, using parametrized function to estimate strategies in games has been at the heart of multi-agent RL~\citep{baxter2000learning,franccois2018introduction,mnih2015human}. The idea of using a latent space to parametrize distributions has also been widely used in the ML community~\citep{goodfellow2014generative}. Our contribution regarding latent matrix games (and more broadly latent RL policies) is theoretical: 
we provide a framework to study equilibria in parametrized NC-NC games as well as an existence result for an equilibrium (Thm.~\ref{thm:non-conv_minimax}).

\paragraph{Bounded rationality.}
In his seminal work,~\citet{simon1969sciences} introduced the principled of bounded rationality. Generally speaking, it aims to capture the idea that rational agents are actually incapable of dealing with the full complexity of a realistic environment, and thus by these limitations, achieve a sub-optimal solution. \citet{neyman1985bounded,papadimitriou1994complexity,rubinstein1998modeling} modeled these limitations by constraining the computational resources of the players. Similarly, Quantal response equilibrium (QRE)~\citep{mckelvey1995quantal} is a way to model bounded rationality: the players do not choose the best action, but assign higher probabilities to actions with higher reward. Overall, QRE, bounded rationality/computation have a similar goal as our notion of limited capacity equilibrium: to model players that are limited by computation/memory/reasoning. However, the way the limits are modeled differs since, in this work, the limitations of the players are embedded in the representative power of the class of function considered. 

\paragraph{Finding a Nash equilibrium of Colonel Blotto.} After its introduction by~\citet{borel1921theorie}, finding a Nash equilibrium of the Colonel Blotto game has been an open question for 85 years. \citet{roberson2006colonel} found an equilibrium solution for the continuous version of the game, later extended to the discrete symmetric case by~\citet{hart2008discrete}. The equilibrium computation in the general case remains open. Recently, Blotto has been used as a challenging use-case for equilibrium computation~\citep{ahmadinejad2019duels}. Similarly, we consider a variant of Blotto to validate that we can practically find approximate equilibrium in games with a single pair of neural networks.

\section{MOTIVATION: MINIMAX IN MACHINE LEARNING}
\label{sec:latent_games_in_ml}
A \emph{two-player zero-sum game} is a \emph{payoff function} $\varphi:\Omega \times \Theta \to \R$, that evaluates pairs of strategies $(w,\theta)$. The goal of the game is to find an \emph{equilibrium}, i.e., a pair of strategies $(\omega^*,\theta^*)$ such that, 
\begin{equation}
     \varphi(w,\theta^*) \leq \varphi(w^*,\theta^*) \leq \varphi(w^*,\theta) ,\; \forall w\in \Omega, \, \theta \in \Theta.  \notag
\end{equation}
Having such an equilibrium ensures that the order in which the players choose their respective strategy does \emph{not} matter and that there is a \emph{non-exploitable} strategy,
\begin{equation}\label{eq:minimax=maximin}
    \min_{\theta \in  \Theta} \max_{w \in \Omega} \varphi(w,\theta) = \max_{w \in \Omega} \min_{\theta \in  \Theta} \varphi(w,\theta) = \varphi(w^*,\theta^*). \notag
\end{equation}
If the function $\varphi$ is concave-convex and if the sets $\Theta$ and $\Omega$ are convex and compact then \citet{sion1958general}'s Minimax theorem guarantees a Nash equilibrium exists. 

Previous theoretical work in the context of machine learning~\citep{arora2017generalization,oliehoek2018beyond,grnarova2017online,hsieh2019finding} considered the model \emph{parameters} $w$ and $\theta$ as the strategies of the game. Arguably, the most well-known example of such a game is GANs.
\begin{example}\citep{goodfellow2014generative} \label{example:standard_gans} A GAN is a game where
the first player picks a binary classifier $D_w$ parametrized by $w \in \R^{p_1}$ called \emph{discriminator}, and the second player picks a \emph{generator} $G_\theta$ that is a parametrized mapping from a latent space to an output space.
The payoff $\varphi$ is then the ability of the first player to discriminate a real data distribution $p_{d}$ from the generated distribution,
\begin{align}\label{eq:standard_gan}
    \varphi(w,\theta) & := 
    \E_{x\sim p_{d}} \big[\log D_w(x)\big]  \notag \\
    &\quad \;+ \E_{z\sim \mathcal{N}(0,I_d)}\big[\log\big(1-D_w(G_{\theta}(z))\big)\big] \,. 
\end{align}
\end{example}
Unfortunately, Example~\ref{example:standard_gans} does \emph{not} satisfy Sion Minimax theorem's assumptions for the following reasons: \begin{enumerate*}[itemjoin = \;\,, label=(\roman*)]
    \item The parameter sets are not compact. 
    \item The function $\varphi$ is not concave-convex because of the non-convexity induced by the neural networks parametrization.
\end{enumerate*}
While one can easily cope with the first issue---by for instance restricting ourselves to bounded weights or by leveraging Fan's Theorem~\citep{fan1953minimax}---the second issue (ii) is an intrinsic part of learning by neural networks.

On the one hand, one cannot expect~\eqref{eq:minimax=maximin} to be valid for general nonconcave-nonconvex (NC-NC) games~\citep{jin2019local}. On the other hand, many games in the context of machine learning have a particular structure since, as we will see in the next section, their NC-NC aspect comes from the neural network parametrization.

\paragraph{Two complementary perspectives on a game.}
Example~\ref{example:standard_gans} can be interpreted as a game between two players, one player, the \emph{generator}, proposes a sample that the other player, the discriminator tries to distinguish from a real data distribution $p_{data}$.
In that game, the parameters $w$ and $\theta$ of the payoff function~\eqref{eq:standard_gan}, do not explicitly correspond to any meaningful strategy--i.e., generating a sample or distinguishing data from generated samples--but they respectively parametrize models (a discriminator and a distribution) that have an intuitive interpretation in the GAN game.  

Considering $q_\theta$ the generated distribution in~\eqref{eq:standard_gan}, 
we then have a duality between {\color{blueblind}parameters} and {\color{orangeblind} models}
\begin{align}
    \tilde \varphi({\color{orangeblind}\underbrace{D_w,q_\theta }_{models}}) &:= \E_{p_{d}} [\log D_w(x)] + \E_{q_\theta }[\log\big(1-D_w(x')\big)
    ] \notag \\[-2mm]
    & = \varphi({\color{blueblind}\underbrace{w,\theta}_{params}}) \,. \label{eq:duality_payoff} \\[-7mm] \notag
\end{align}
A compelling aspect of this dual perspective is that even though, one \emph{cannot} expect $\varphi$, the payoff function of the {\color{blueblind}parameters} $w$ and $\theta$, to be concave-convex, the payoff of the {\color{orangeblind} models} $\tilde \varphi$ is \emph{concave-convex}.  Formally, for $w_i \in \Omega$, $\theta_i \in \Theta$, and $\lambda_i \in [0,1] \,,\, i =1\ldots K \,,\,\lambda_1+\ldots+ \lambda_K = 1$ we have (by concavity of $\log$ and linearity of $q \mapsto \E_{q}$),
\begin{align*}
    & \textstyle
    \tilde \varphi(\sum_{i=1}^K\lambda_i D_{w_i} ,q_\theta) 
    \geq  \sum_{i=1}^K\lambda_i\tilde\varphi( D_{w_i} ,q_\theta) \,,\, \forall \theta \in \Theta\,, \label{eq:payoff_concave}  
    \\
    &\textstyle \tilde\varphi(D_w , \sum_{i=1}^K\lambda_iq_{\theta_i} ) 
    = \sum_{i=1}^K\lambda_i \tilde\varphi(D_w ,q_{\theta_i}) \,,\, \forall w \in \Omega\,.
\end{align*}
Note that the notion of convex combination for the {\color{orangeblind} models} is quite subtle here: $\lambda D_w + (1-\lambda) D_{w'}$ corresponds to a convex combination of functions while $\lambda q_\theta + (1-\lambda) q_{\theta'}$ corresponds to a convex combination (a.k.a, mixture) of distributions.

Even if the payoff~\eqref{eq:duality_payoff} is concave-convex with respect to $(D,p)$, one \emph{cannot} apply (yet) any standard minimax theorem for the following reason: given $w_1,w_2 \in \Omega$ and $\lambda \in [0,1]$ we may have
\begin{equation}
    \nexists w \in \Omega \,,\quad \text{s.t.} \quad \lambda D_{w_1} + (1-\lambda) D_{w_2} = D_{\omega} \,, \label{eq:convexity_D}
\end{equation}
meaning that the set of functions $\gF_\Omega:=\{D_w \,|\, w \in \Omega\}$ may \emph{not} be convex. However, for the particular case of parametrized neural networks we will show that the set $\gF$ is ``almost convex'' (see Prop.~\ref{prop:conv_distrib} and~\ref{prop:conv_function}). It is one of the core results used in Thm.~\ref{thm:non-conv_minimax}'s proof.

\section{AN ASSUMPTION FOR NC-NC GAMES}

The games arising in machine learning are not classical normal- or extensive-form games. Rather, they often use neural nets models to approximate complex functions and high dimensional distributions~\citep{brock2019large,razavi2019generating}. That is why they are often considered \emph{general nonconcave-nonconvex (NC-NC) games}~\eqref{eq:nonconcave-nonconvex}. 
However, as illustrated in~\eqref{eq:duality_payoff}, in the machine learning context, many games have a particular structure where the models' payoff is concave-convex. 
\begin{assumption}\label{assump:nonconvex-nonconcave_convex}
The NC-NC game~\eqref{eq:nonconcave-nonconvex} is assumed to have a \emph{concave-convex} models' payoff, i.e., {\color{blueblind} $w$ and $\theta$} respectively parametrize {\color{orangeblind} $f_w$ and $q_\theta$} such that,
\begin{equation}
    \varphi({\color{blueblind}\underbrace{w,\theta}_{params}}) = \tilde \varphi({\color{orangeblind}\underbrace{f_w,q_\theta }_{models}}) 
\end{equation}
where $(f,q) \mapsto \tilde\varphi(f,q)$ is concave-convex.
We call $f_w$ and $q_\theta$ {\color{orangeblind} the models picked by the players}, they can either be 
a parametrized function or distribution.
\end{assumption}
The intuition behind this assumption is that the nonconvex-nonconcavity of the problem comes from the (neural network) parametrization of the models. 
One example of such a game has been developed in~\eqref{eq:duality_payoff} where the first player picks a function $D_w$ and the second one a distribution over images $q_\theta$. Another closely related example is the Wasserstein GAN (WGAN).
\begin{example}\citep{arjovsky2017wasserstein} \label{example:WGAN}The WGAN formulation is a minimax game with a payoff $\varphi$ s.t.,
\begin{equation} \label{eq:wgan}
    \varphi(w,\theta)=\tilde \varphi(D_w,q_\theta) := 
    \E_{x\sim p_{data}}  D_w(x) - \E_{x' \sim q_\theta}D_w(x') \notag
\end{equation}
where the discriminator $D_w$ has to be 1-Lipschitz, i.e., $\|D_w\|_L \leq 1$. By bilinearity of the function $(D,p)\mapsto \E_p[D(x)]$ we have that $\tilde \varphi$ is bilinear and thus satisfies Assumption~\ref{assump:nonconvex-nonconcave_convex}.
\end{example}
Finally, we present how Assumption~\ref{assump:nonconvex-nonconcave_convex} holds when trying to solve a matrix game with a very large (or even infinite) number of strategies by parametrizing mixed strategies.\footnote{Note that here we do not claim the novelty of parametrizing policies/strategies, such idea has been used in many games and RL applications (see related work section). We instead focus on illustrating how a particular way of parametrizing leads to game that satisfies Assump.~\ref{assump:nonconvex-nonconcave_convex}.}

\paragraph{Using function approximation to solve matrix games} In the case of matrix games, the payoff function $\varphi: A \times B \to \R$ has no concave-convex structure, and the sets $A$ and $B$ are often even discrete. \citet{von1944theory} introduced mixed strategies $p \in \Delta(A)$, where $\Delta(A)$ is the set of probability distributions over $A$, in order to guarantee the existence of an equilibrium. In game-theory, a well-known example of a challenging matrix game is the Colonel Blotto game.
\begin{example}[Colonel Blotto Game] \label{example:Blotto}
Consider two players who control armies of $S_1$ and $S_2$ soldiers respectively. Each colonel allocates their soldiers on $K$ battlefields. A strategy for player-$i$ is an allocation $a_i \in A_i$ and the payoff of the first player is the number of battlefields won  
\begin{equation}\label{eq:strat_blotto}
\textstyle
 \varphi(a_1,a_2) := \frac{1}{K}\sum_{k=1}^K \mathbf{1}\{[a_1]_k > [a_2]_k\} 
\end{equation}
where $A_i :=  \Big\{ a \in \sN^K \,:\, \sum_{k=1}^K [a]_k \leq S_i \,,\; 1\leq k\leq K \Big\}$ and $[a]_k$ correspond to the $k^{th}$ coordinate of $a$.
\end{example}
 In Example~\ref{example:Blotto}, the number of strategies grows \emph{exponentially fast} as $K$ grows. Consequently, one cannot afford to work with an explicit distribution over the strategies. A tractable way to compute an equilibrium of the Colonel Blotto Game has been an open question for decades. The GANs examples (Example~\ref{example:standard_gans} \& \ref{example:WGAN}) suggest to consider distributions implicitly defined with a generator. Given a latent space $\mathcal Z$, a latent distribution $\pi$ on $\mathcal Z$ and a mapping $g_\theta:\mathcal Z \to A$, we can define the  distribution $q_\theta \in \Delta(A)$ as
\begin{equation}\label{eq:mapping_distribution}
    a \sim q_\theta\;: \;  a= g_\theta(z) \,,\; z \sim \pi  \,.
\end{equation}
\begin{definition}[Latent Matrix Game]
\label{def:latent_game}
A \emph{latent matrix game} $(\varphi, \mathcal{F}, \mathcal{G})$ is a two-player zero-sum game where the players pick $f_w \in \mathcal{F}$ and $g_\theta \in \mathcal{G}$ and, given $\pi$ and $\pi'$ two fixed distributions, obtain payoffs
\begin{equation}
    \varphi(w,\theta) := \E_{z \sim \pi,\, z' \sim \pi'} \big[\varphi\big(f_w(z),g_\theta(z')\big)\big] \notag \,.
\end{equation}
\end{definition} 

The reformulation of any matrix game as a latent game satisfies Assumption~\ref{assump:nonconvex-nonconcave_convex}.
\begin{repexample}{example:Blotto}[Latent Blotto] Consider the functions $f_w: \R^p \to A_1$ and $g_\theta: \R^p \to A_2$. The payoff is 
\begin{equation}\label{eq:payoff_latent_blotto}
\textstyle
    \varphi(w,\theta) := \frac{1}{K}\sum_{k=1}^K \sP\big( [f_w(Z_1)]_k > [g_\theta(Z_2)]_k\big)
\end{equation}
where $Z_1, Z_2 \sim \mathcal{N}(0,I_p)$ are independent Gaussians and $A_i$ is defined in~\eqref{eq:strat_blotto}.
\end{repexample}

Latent matrix games encompass multi-agent RL games played with RL policies such as the setting used by~\citet{vinyals2019grandmaster} to play StarCraft II.
The agent, called AlphaStar, has a latent-conditioned policy $\pi(a|s,z)$ where $z$ belongs to a structured space that represents information about how to start constructing units and buildings, and that is sampled from an expert human player distribution: $z \sim p_{human}(z)$.  Given two agents $\pi_1(a|s,z)$ and $\pi_2(a|s,z)$, the payoff in the latent game is
    $\varphi(\pi_1,\pi_2) = \sP(\pi_1 \;\text{beats} \; \pi_2)\,.$
The classes $\mathcal{F}$ and $\mathcal{G}$ correspond to the neural architectures used to parametrize the policies; the priors $\pi$ and $\pi'$ are the human expert distribution $p_{human}$. 

In that example, and more generally in multi-agent RL zero-sum games played with policies parametrized by neural networks, the payoff  $\varphi(w,\theta) = \sP(\pi_w \;\text{beats} \; \pi_\theta)$ is a potentially nonconcave-nonconvex function of the parameters but satisfies Assumption~\ref{assump:nonconvex-nonconcave_convex}.

\section{A MINIMAX THEOREM}
\label{sec:minimax_latent}

We want to prove a minimax theorem for some NC-NC games~\eqref{eq:nonconcave-nonconvex} that satisfy Assumption~\ref{assump:nonconvex-nonconcave_convex}. We start with an informal statement of our result.
\begin{theorem}\label{thm:non-conv_minimax}[Informal]
Let $\varphi$ be a nonconcave-nonconvex payoff that satisfies Assump.~\ref{assump:nonconvex-nonconcave_convex} with $\tilde \varphi$ bilinear and where the players pick $(w,\theta)$, the parameters of two neural networks.
For any $\epsilon >0$ there exists $(w_\epsilon,\theta_\epsilon)$ achieving an approximate \emph{limited-capacity equilibrium}.
\end{theorem}
The notion of approximate limited-capacity equilibrium mentioned in that informal statement is detailed in the complete statement of Theorem~\ref{thm:non-conv_minimax}.
When played with neural networks Example~\ref{example:WGAN} and~\ref{example:Blotto} satisfy the hypothesis of this theorem.
The proof of this Theorem is split into 3 main steps:
\begin{enumerate*}[itemjoin = \;\,, label=(\roman*)]
        \item in \S\ref{sub:equilibirum convex hull} by using the fact that $\varphi(w,\theta) = \tilde \varphi(f_w,q_\theta)$ we provide the existence of a limited-capacity equilibrium in the convex hull of the models' space of $f_w$ and $q_\theta$. Note that, since we are working in the convex hull, one can \emph{only} expect (in general) to achieve this equilibrium with a collection of parameters $(w_i,\theta_i)_{i \in I}$.
        \item in \S\ref{sub:epsilon_minimax} we show that approximate equilibrium can be achieved with a relatively small convex combination.
        \item  in~\S\ref{sec:pure_nash} we show that when using neural networks, such small convex combination of models can be achieved by a single larger neural network.
\end{enumerate*}
A formal definition of convex combination of models is provided in \S\ref{sub:equilibirum convex hull}.
\subsection{Limited Capacity Equilibrium in the Models' Space}
\label{sub:equilibirum convex hull}
Recall that by Assump.~\ref{assump:nonconvex-nonconcave_convex}, the NC-NC payoff $\varphi$ can be written as, 
\begin{equation}\label{eq:function_params}
    \varphi({\color{blueblind}w,\theta}) = \tilde \varphi({\color{orangeblind}f_w,q_\theta}) 
\end{equation}
where $(f,q) \mapsto \tilde\varphi(f,q)$ is concave-convex.
The models {\color{orangeblind}$f_w$} and {\color{orangeblind}$q_\theta$} are either {\color{orangeblind}functions or distributions} respectively {\color{blueblind}parametrized by $w$ and $\theta$}. For instance, in the context of WGAN (Example~\ref{example:WGAN}), $f_w$ would be the discriminator and $q_\theta$ would be the generated probability distribution. In that example, notice that $\tilde \varphi(f_w,\cdot)$ is \emph{not} convex with respect to the generator function but only with respect to the generated distribution. Similarly, if we computed convex combinations generator's parameters, the payoff $\varphi$ would not be convex in general. Moreover, the Lipchitz constraint in Example~\ref{example:WGAN} is a natural constraint in the {\color{orangeblind}function space}, but it is challenging to translate it into a constraint in the {\color{blueblind} parameter space}.\footnote{In practice, parameters are clipped~\citep{arjovsky2017wasserstein} or the Lipchitz constant of the network is approximated~\citep{miyato2018spectral}. These approximations can be arbitrarily far from the original constraint.}  Overall, using~\eqref{eq:function_params} one can rewrite~\eqref{eq:nonconcave-nonconvex} as follows,
\begin{equation}\label{eq:functional game}
    \min_{f \in \gF_\Omega} \max_{g\in \gG_\Theta} \tilde \varphi(f,g)
\end{equation}
where $\gF_\Omega$ and $G_\Theta$ are function or distribution spaces (depending on the application) incorporating the limited capacity constraints of the problem, e.g., Lipschitz constraint. \textbf{In the following}, for simplicity of the discussion, we discuss what the formal definitions of a convex combination are when {\color{orangeblind}$\gF_\Omega$ is a function space} and {\color{orangeblind} $\gG_\Theta$ is a distribution space} when we have no additional constraint aside from the parametrization, i.e., $\gF_\Omega:= \{f_w \;|\; w \in \Omega\}$ and $\gG_\Theta := \{q_\theta \,|\, \theta \in \Theta\}$. However, these notions and our results extend if we consider that both models are distributions (e.g., in Example~\ref{example:Blotto}), or if we add any convex constraint on the functions or the distributions, see Example~\ref{example:WGAN}.

\paragraph{Convex combination of functions.} 
Let us consider $w_1$ and $w_2 \in \Omega$, the convex combination of the models $f_{w_1}$ and $f_{w_2}$ is their point-wise averaging.
The convex hull of $\gF_\Omega$ can be defined as,
\begin{align}\label{eq:convFunc}
    \hull(\gF_\Omega)
    &:= \{ \text{Averages from } \gF_\Omega\} \\
    &= \Big\{ \sum_{i=1}^K \lambda_i f_{w_i} \,|\, w_i \in \Omega,\, \lambda \in \Delta_K, K \geq 0\Big\}.  \notag
\end{align}
where $\Delta_K$ is the $K$-dimentional simplex, i.e., $\{\lambda \in \sR^K\,,\, \lambda_i \geq 0 \,,\, \sum_{i=1}^K \lambda_i = 1\}$.
\paragraph{Convex combination of distributions.}
Consider latent mappings $\theta_1$ and $\theta_2 \in \Theta$ that parametrize probability distribution $q_{\theta_1}$ and $q_{\theta_2}$ over a set $\gX$. The \emph{convex combination} $q_\lambda$ of $q_{\theta_1}$ and $q_{\theta_2}$ with $\lambda \in [0,1]$ is the mixture of these two probability distributions, $q_{\lambda} := \lambda q_{\theta_1} + (1 - \lambda) q_{\theta_2}\,.$

To sample from $q_\lambda$, flip a biased coin with $\sP(\mathrm{heads})=\lambda$. If the result is $\mathrm{heads}$ then sample a strategy from $q_{\theta_1}$ and if the result is $\mathrm{tails}$ then sample from $q_{\theta_2}$.
The convex hull of $\gG_\Theta$ is, 
\begin{align}\label{eq:mixtures}
    \hull(\gG_\Theta)
    &:= \{ \text{Mixtures from } \gG_\Theta\} \\
    &= \Big\{ \sum_{i=1}^K \lambda_i q_{\theta_i} \,|\, \theta_i \in \Theta,\, \sum_{i=1}^K \lambda \in \Delta_K ,\, K\geq 0\Big\}. \notag
\end{align}
The set $\hull(\gG_\Theta)$ is a subset of $\mathcal P(\gX)$, the set of probability distributions on $\gX$. This set is different from the set of distributions supported on $\gG_\Theta$ considered by~\citet{arora2017generalization,hsieh2019finding}. It contains `smaller' mixtures because there may be many distributions supported on $\gG_\Theta$ that correspond to the same $p \in \hull(\gG_\Theta)$. Moreover these works did not take advantage of the convexity with respect to the discriminator function (see Example~\ref{example:standard_gans} and~\ref{example:WGAN}) by considering~\eqref{eq:convFunc}.

\paragraph{Existence of an equilibrium by playing in the convex hulls.}
Our first result is that there exists an equilibrium by allowing functions or distributions to be picked from their convex hulls.

\begin{restatable}{proposition}{minimaxHull} \label{prop:minimax_hull}
Let $\varphi$ be a game that follows Assumption~\ref{assump:nonconvex-nonconcave_convex}. If $\gG_\Theta$ and $\gF_\Omega$ are compact, then there exist a \emph{value} for the game such that, 
\begin{align}\label{eq:minmax_latent} 
     \Val(\Omega,\Theta) 
     &:=\sup_{f \in  \hull(\gF_\Omega)} \inf_{q \in \hull(\gG_\Theta)} \tilde \varphi(f,q) \notag \\
    &=  \inf_{q \in \hull(\gG_\Theta)} \sup_{f \in \hull(\gF_\Omega)} \tilde \varphi(f,q) \,,
\end{align}
where $\hull(\gG_\Theta)$ and $\hull(\gF_\Omega)$ are either defined in~\eqref{eq:convFunc} or in~\eqref{eq:mixtures}, depending on the type model.
\end{restatable}
After showing that the closure of $\hull(\gG_\Theta)$ and $\hull(\gF_\Omega)$ are compact, this proposition is a corollary  of~\citet{sion1958general}'s minimax theorem (see~\S\ref{app:proof}). 
Note that $\Omega$ and $\Theta$ are arbitrary and that this equilibrium differs from the infinite-capacity equilibrium of the game~\eqref{eq:functional game} where we would allow $f$ and $g$ to be any function or distribution (i.e. with no parametrization restriction).
Because we consider the convex hull of $\gF_\Omega$ and $\gG_\Theta$, this equilibrium is achieved with \emph{convex combinations} (Eq.~\ref{eq:convFunc} \&~\ref{eq:mixtures}) of $q_{\theta_i}\,, \, i \geq 0$ (resp. $f_{w_i}$) and thus there is no reason to expect to achieve this equilibrium with a single pair of weights $(w,\theta)$ in general. However in \S\ref{sub:epsilon_minimax}, we show that one can approximate such an equilibrium with relatively small convex combinations.

\subsection{Approximate minimax equilibrium}
\label{sub:epsilon_minimax}

Approximate equilibria for~\eqref{eq:minmax_latent} are the pairs of models $\epsilon$-close to achieving the value of the game.



\begin{definition}[$\epsilon$-equilibrium]\label{def:latent_eqm} A pair $(f_\epsilon^*,q_\epsilon^*) \in \hull(\gF) \times \hull(\gG)$ is an $\epsilon$-equilibrium if,
\begin{align} \notag
    &\min_{q \in \hull(\gG_\Theta)} \tilde \varphi(f^*_\epsilon,q)  \geq \Val(\Omega,\Theta) - \epsilon \\
    &\text{and} \quad 
    \max_{f \in \hull(\gF_\Omega)} \tilde \varphi(f,q^*_\epsilon) \leq \Val(\Omega,\Theta)+ \epsilon \,.
\end{align}
\end{definition}
Note that $f^*_\epsilon$ does not depend on $p^*_\epsilon$ and vice-versa.
We will show that such approximate equilibria are achieved with finite convex combinations. Considering $f_k \in \gF_\Omega$ and $q_k' \in \gG_\Theta$ (that can either be functions or distributions) we aim at finding the smallest convex combination that is an $\epsilon$-equilibrium. 
\begin{align}
\textstyle
   &(K_\epsilon^{\Omega},K_\epsilon^{\Theta}) := \;\text{Smallest} \; K \text{ and } K'\in \sN \notag\\ 
   &\text{s.t.} \;\,
   (\sum_{k=1}^{K} \lambda_k f_k,\sum_{k=1}^{K'} \lambda'_k q_k) \;\text{is an $\epsilon$-equilibrium}. 
\end{align}
Our goal is to provide a bound that depends on $\epsilon$ and on some properties of the classes $\gF_\Omega$ and $\gG_\Theta$. 

\begin{restatable}{theorem}{FiniteKEpsilon}\label{thm:finiteKepsilon}
Let $\varphi$ a game that satisfies Assumption~\ref{assump:nonconvex-nonconcave_convex}. 
If $\tilde \varphi$ is bilinear, $\|\theta\| \leq R\,,\, \|w\|\leq R\,,\, \forall w,\theta \in \Omega\times \Theta \subset \R^d \times \R^p $ and $\varphi$ is $L$-Lipschitz then, 
\begin{equation}
    K_\epsilon^\Omega  
    \leq \tfrac{4D_w^2 p}{\epsilon^2}\ln(\tfrac{6 R L}{\epsilon^2})
    \;\; \text{and} \;\;
    K_\epsilon^\Theta  
    \leq \tfrac{4D_\theta^2 d}{\epsilon^2}\ln(\tfrac{6 R L}{\epsilon^2})
\end{equation}
where $D_w:= \max_{w,w',\theta} \varphi(w,\theta) - \varphi(w',\theta)$ and $D_\theta := \max_{w,\theta, \theta'} \varphi(w,\theta) - \varphi(w,\theta') $. 
\end{restatable}

Roughly, the number $K_\epsilon^\Theta$ expresses to what extent the set of distributions induced by the mappings in $\gG_\Theta$ has to be `convexifed' to achieve an approximate equilibrium.
Note that in practice we expect this quantity to be small. For instance, in the context of GANs, if the class of discriminators $\gF_\Omega$ contains the constant function $D(\cdot)=.5$ then $K_\epsilon^{\Omega} = 1$ since $\varphi(D,G) = 0 \,,\, \forall G \in \gG$. 

Two close related results are~\citep[Theorem 4]{Farnia2018convex} and~\citep[Theorem 4.2]{arora2017generalization}. In the former, the authors point out that the functional space of neural networks is not convex and prove that any function in the convex hull of neural networks can be approximated by a finite average of function. In the latter, the authors show that by considering uniform mixtures of generators and disciminators one can approximate the equilibrium of a GAN. Our results generalize these two ideas by showing that for any payoff functon~$\varphi$ that satisfy Assumption~\ref{assump:nonconvex-nonconcave_convex} one can find an approximate equilibrium with a finite convex combination of function or distributions (depending on the context). We thus simultaneously handle the case of convex combinations of functions or distributions and consider a more general class of payoff functions.

\subsection{Achieving a Mixture or an Average with a Single Neural Net}
\label{sec:pure_nash}
We showed above that under the assumption of Theorem~\ref{thm:finiteKepsilon}, approximate equilibria can be achieved with finite convex combinations. In this section, we investigate how it is possible to achieve such approximate equilibria with a single neural network.
Formally, a neural network $g: \R^{d_{in}} \to \R^{d}$ can be written as,
\begin{equation}
\textstyle
    g_\theta(x) = b_l + W_l\sigma (b_{l-1}+W_{l-1} \ldots \sigma (b_1+W_1 x)) \,,
    \label{eq:ReLU}
\end{equation}
where $W_1,\ldots,W_l$ are the weight matrices, $b_l,\ldots,b_1$ the biases and $\sigma$ is a given non-linearity. We note $\theta$ the concatenation of all the parameters of this neural network.
We present two results on the geometry of the space of neural networks. 
The first one concerns \emph{mixtures of distributions} represented by latent neural nets with ReLU non-linearity, and the second one concerns \emph{convex combinations} of neural nets as functions.

\paragraph{Neural Nets Represent Mixtures of Sub-Nets.}
First, we get interested in the probability distributions $q_\theta$ induced by $g_\theta \,,\, \theta \in \Theta$, defined as
\begin{equation}
  a\sim q_\theta \; : \; a = g_\theta(z) \;\; \text{where} \;\; z\sim U([0,1]) \,. \label{eq:induced_distribution} 
\end{equation}
One of the motivations of this work is to represent distribution over images usually represented by a high dimensional vector in $[0,1]^d$. That is why we will assume that our generator function take its value in $[0,1]^d$. Moreover, for this proposition, we only consider the ReLU non-linearity $\sigma(x) = \relu(x) := \max(0,x)$.

\begin{restatable}{proposition}{convexNetsDistribution}
\label{prop:conv_distrib} Let $\theta_k \in [-R,R]^p,\, k =1\ldots K,$ be the parameters of $k$ ReLU nets with $p$ parameters. If the input latent variable is of dimension $d_{in}$ = 1 and if for all $k=1\ldots K\,,\, z \in [0,1]$, $g_{\theta_k}(z) \in [0,1]^d$ and $g_{\theta_k}$ is constant outside of $[0,1]$, then there exists a ReLU net with $K(p+6)$ non-linearities $\theta \in [-KR,KR]^{K(p+6)}$ such that $d_{TV}( \tfrac{1}{n} \sum_{k=1}^K q_{\theta_k}, q_\theta) \leq 1/R$ where $d_{TV}$ is the total variation distance.
\end{restatable} 

Fig.~\ref{fig:preMappingAvg} (in \S\ref{app:proof}) illustrates how $g_{\theta}$ is constructed.
Unlike the universal approximation theorem, Prop.~\ref{prop:conv_distrib} shows that \emph{a single neural network} can \emph{represent} mixtures. On the one hand, when one wants to approximate an arbitrary continuous function, the number of required hidden units may be prohibitively large~\citep{lu2017expressive} as the error $\epsilon$ vanishes. On the other hand, the dimension of $\theta$ in Prop.~\ref{prop:conv_distrib} does not depend on any vanishing quantity. 
The high-level insight is that a large enough ReLU net can represent mixtures of distributions induced by smaller ReLU nets, with a width that grows linearly with the size of the mixture.

\paragraph{Neural Nets can represent an average of Sub-Nets.}
If we consider averages of functions as described in \eqref{eq:convFunc}, we can show that point-wise averages of neural networks can be represented by a wider neural network. This result is valid for any non-linearity.
\begin{restatable}{proposition}{convexNetsFunction}
\label{prop:conv_function} Let $w_k \in [-R,R]^p\,,\, k = 1\ldots K$ be the parameters of $k$ neural nets with $p$ parameters, there exists au neural net parametrized by $w\in [-R,R]^{Kp}$, such that $\tfrac{1}{n} \sum_{k=1}^K f_{w_k} = f_{w}$.
\end{restatable}
Figure~\ref{fig:convex_function} shows how $f_w$ is constructed. 
Similarly as the Prop.~\ref{prop:conv_distrib}, Prop.~\ref{prop:conv_function} is a representation theorem that shows that the space or neural networks with a fixed number of parameters is `almost' convex.


\subsection{A Minimax Theorem for NC-NC Games Played with Neural Nets}
\label{sub:neural_net_minimax}

Prop.~\ref{prop:conv_distrib} and~\ref{prop:conv_function} give insights about the representative power of neural nets: as their number of parameters grows, neural nets can express larger mixtures/averages of sub-nets.  
Combining these properties with Thm.~\ref{thm:finiteKepsilon}, we show that approximate equilibria can be achieved for such nonconcave-nonconvex (NC-NC) payoff.

\begin{reptheorem}{thm:non-conv_minimax}
Let $\varphi$ be a NC-NC game that follows Assumption~\ref{assump:nonconvex-nonconcave_convex}.
We assume that the models' payoff $\tilde \varphi$ is bilinear and consider two cases where the models have $p$ parameters bounded by $R$, i.e., $(w,\theta) \in [-R,R]^{2p}$:
\begin{enumerate}
    \item \emph{[Function vs. \!distribution]}: $\varphi(w,\theta) = \tilde \varphi(f_w, q_\theta)$ where $q_\theta$ is a distribution parametrized by a ReLU net with $d_{in}=1$ (see Eq.~\ref{eq:induced_distribution}), and $f_w$ is a neural network with any non-linearity. Applies to Ex.~\ref{example:WGAN}.
    \item \emph{[Distribution vs. \!distribution]}: $\varphi(w,\theta) = \tilde \varphi(q_w, q_\theta)$ where $q_\theta$ and $q_w$ are distributions parametrized by a ReLU net with $d_{in}=1$. Applies to Example~\ref{example:Blotto}.
\end{enumerate}
If $\tilde \varphi$ is $\tilde L$-Lipschitz, and if $\varphi$ is $L$-Lipschitz, then for any $\epsilon >0$, there exists $(w_\epsilon^*,\theta_\epsilon^*) \in [-R,R]^{2p}$, such that,
\begin{equation} \label{eq:bound_p_epsilon}
    \min_{\theta \in [-R,R]^{p_\epsilon} } \!\varphi(w_{\epsilon}^*,\theta) + \epsilon + \frac{2\tilde L}{R} \geq \!\max_{w \in [-R,R]^{p_\epsilon}} \! \varphi(w,\theta_\epsilon^*) \,, \notag
\end{equation}
where $p_\epsilon \geq C \epsilon\sqrt{\frac{p}{\log(R\sqrt{p}/\epsilon)}}$, $R_\epsilon \geq R\frac{p_\epsilon}{p}$, and the subnetworks generating the distributions (see Eq.~\ref{eq:induced_distribution}) takes their values in $[0,1]^d$ and are constant outside of $[0,1]$.
\end{reptheorem}
An explicit formula for $C$ is provided in \S\ref{app:minimax} as well as variants of this theorem when both players pick a model that is a function.
Theorem~\ref{thm:non-conv_minimax} shows the existence of a notion of weaker-capacity-\emph{equilibrium} for a nonconcave-nonconvex game where the models use a \emph{standard fully connected architecture}. This result differs from~\citet[Theorem 4.3]{arora2017generalization} who, only in the context of GANs, design a \emph{specific architecture} to achieve a different notion of approximate equilibrium.

The notion of weaker-capacity is encompassed within the fact that $w_\epsilon$ and $\theta_\epsilon$ are of dimension $p_\epsilon \leq p$ and are bounded by $R_\epsilon \leq R$. Roughly, this theorem can be interpreted as follows: if one considers a minimax game where the players pick neural networks and where the payoff is bilinear with respect to the models (functions on distributions) represented by these neural networks, then there exist a combination of parameters that can achieve an $\epsilon$-approximate minimax against subnetworks with $p_\epsilon$ parameters, i.e., the first (resp. second) player cannot be beaten by more than $\epsilon$ by any sub-network of the second (resp. first) player.

Regarding the number of parameters $p_\epsilon$ of these subnetworks, on the one hand, if $\epsilon\sqrt{p}<1$, then the lower-bound on $p_\epsilon$ is vacuous, on the other, the number of parameters of the higher-capacity networks $p$ only needs to (roughly) grow \emph{quadratically} with $\epsilon$ to achieve a non-vacuous bound. Hence, a consequence of Theorem~\ref{thm:non-conv_minimax} is that, for the NC-NC games that follow the assumptions of Theorem~\ref{thm:non-conv_minimax}, \emph{highly over parametrized networks} can provably achieve a non-vacuous notion of approximate equilibrium. 

It is worth discussing what happens when some assumptions fail to hold. We think the Lipschitz assumptions are quite standard. We focus on the case of distributions parametrized by a ReLU net with $d_{in} =1$. We think the assumptions that the non-linearities must be ReLU is an artefact form the proof technique and we believe that with a bit more work a similar version of Theorem~\ref{thm:non-conv_minimax} would hold with any non-linearities with which one can arbitrarily approximate indicator functions. Extending theorem to $d_{in}>1$ seems possible though very challenging. The difficulty relies on extending Proposition~\ref{prop:conv_distrib} to $d_{in}\geq 2$. 
One idea, may be to consider deeper networks to represent a mixture of smaller nets. We leave this question open.

\section{APPLICATION: SOLVING COLONEL BLOTTO GAME}
\label{sec:blotto}
\begin{figure*}[h]
    \centering
    \begin{subfigure}{.6\linewidth}
    \centering
    \hspace{-1cm}
    \includegraphics[width=.95\linewidth]{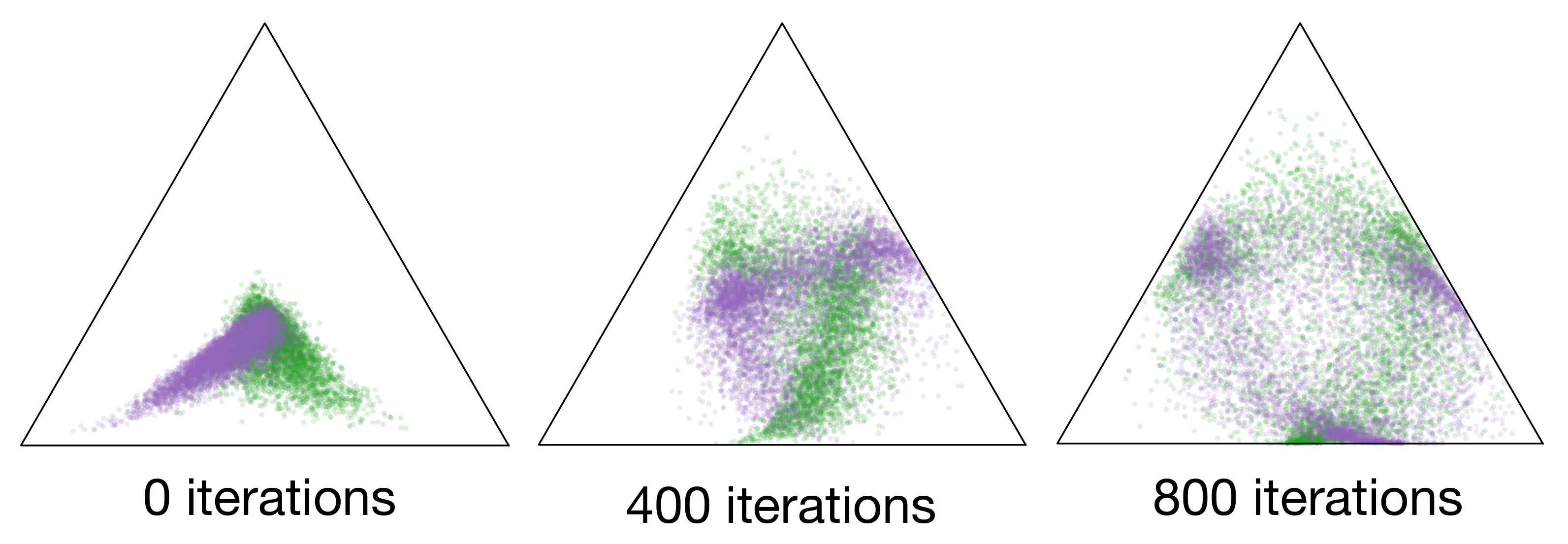}
    \caption{ 
    $5000$ samples using the latent mapping $f$ and $g$ after $0,400$, and $800$ training steps. Their respective suboptimality along training has a value of $1.5,1.2,$ and $.5$.
    }
    \label{fig:sample}
    \end{subfigure}
    \quad 
    \begin{subfigure}{.35\linewidth}
    \centering
    \includegraphics[width=.95\linewidth]{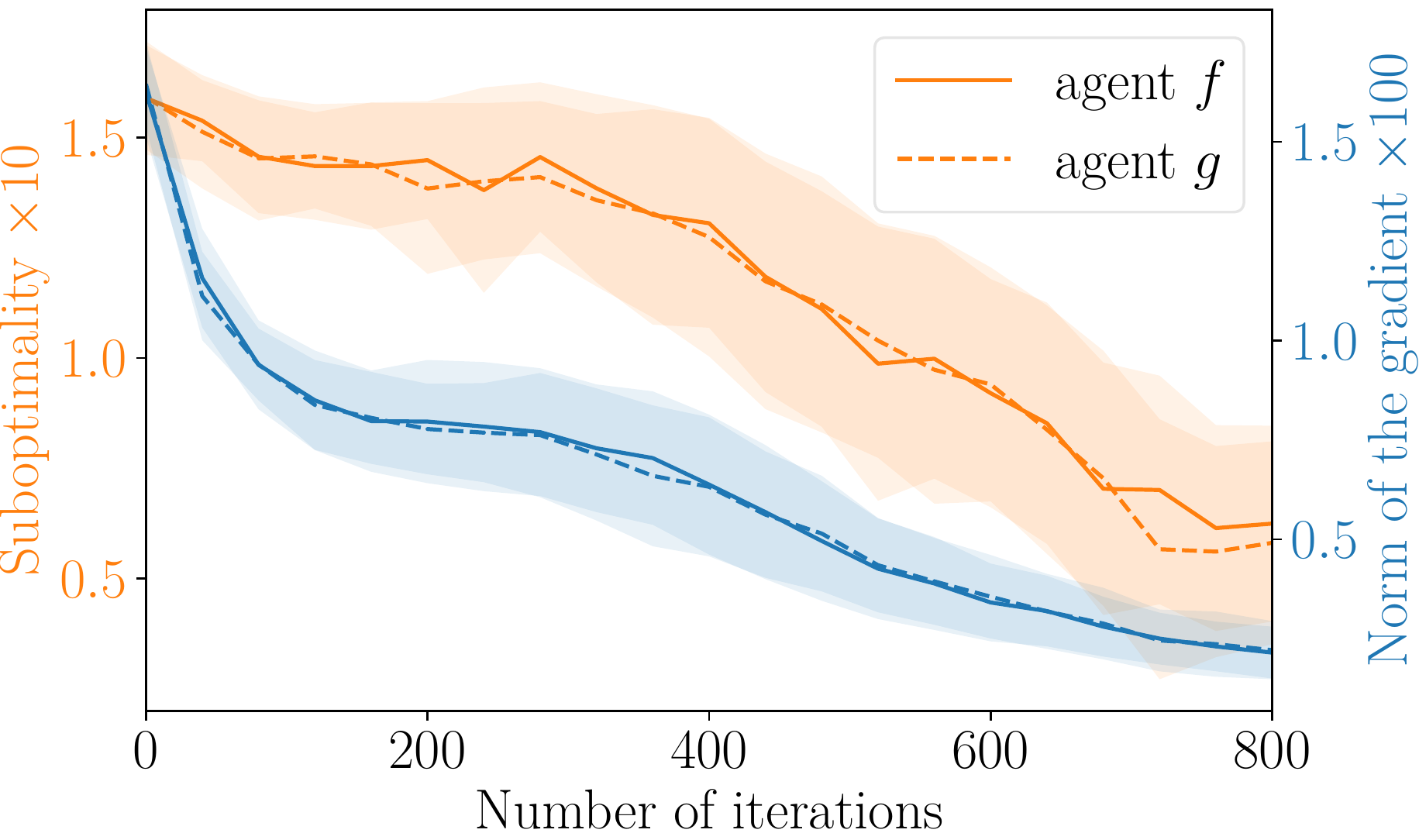}
    \caption{ Performance and convergence of the agents. 
    }
    \label{fig:performance}
    \end{subfigure}
    \caption{
    Training of latent agents to play differentiable Blotto with $K=3$. 
    \textbf{Right:} The suboptimality corresponds to the payoff of the agent against a best response. The curves correspond to averages over 40 random seeds with standard deviation.}
\end{figure*}

Finally, we present some experiments with $d_{in}\geq2$ on the Colonel Blotto Game to illustrate the soundness of Theorem~\ref{thm:non-conv_minimax} and explore whether equilibria can be achieved with larger input spaces. We apply the latent game approach developed in Definition~\ref{def:latent_eqm} to solve a differentiable version of Example~\ref{example:Blotto}. 
We consider a continuous relaxation of the strategy space where $S_1 = S_2$. After renormalization we have that $A_1 = A_2 = \Delta_K$, where $\Delta_K$ is the $K$-dimensional simplex. It is important to notice that in that case an \emph{allocation} corresponds to a \emph{point on the simplex} and a \emph{mixture of allocation} corresponds to a \emph{distribution over the simplex}.
We replace the payoff~\eqref{eq:payoff_latent_blotto} of Latent Blotto by a differentiable one,
\begin{equation}\label{eq:diff_blotto}
      \varphi(w,\theta) :=  \E_{\substack{z \sim \pi \\z' \sim \pi'}}\big[\frac{1}{K}\sum_{k=1}^K \sigma ([f_w(z) - g_\theta(z')]_k)\big]
\end{equation}
where $\sigma$ is a sigmoid minus $1/2$ and $f_w,g_\theta:\R^p \to \Delta_K$. This game has been theoretically analyzed by~\citet{ferdowsi2018generalized} when $S_1 > S_2$.

The generative functions $f_w$ and $g_\theta$ are dense ReLU nets with $4$ hidden layers, $16$ hidden units per layer, and a $K$-dimensional softmax output with a Gaussian latent variable ($d_{in}=16$). We trained the models using gradient descent ascent on the parameters of $f$ and $g$ with the Adam optimizer~\citep{kingma2014adam} with $\beta_1=.5$ and $\beta_2 = .99$.

In Fig.~\ref{fig:performance}, we present the performance of the agents against a best response. To compute it, we sampled $5000$ strategies and computed the best response against this mixed strategy using gradient ascent on the simplex. We also computed the norm of the (stochastic) gradient used to update $f$. In Fig.~\ref{fig:sample}, we plotted samples from $f$ at different training times. As we get closer to convergence to a non-exploitable strategy, we can see that this distribution avoids the center of the simplex (putting troops evenly on the battlefields) and the corners (focusing on a single battlefield) that are strategies easily exploitable by focusing on two battlefields, this correlates with the decrease of the gradient and of the suboptimality indicating that the agents learned how to play Blotto.

\section{DISCUSSION}
\vspace{-2mm}
Nonconcave-nonconvex games radically differ from minimization problems since equilibria may not exist in general. 
How, then, can neural nets regularly find meaningful solutions to games like GANs? 

In this work, we partially answer this question by leveraging the structure of GANs to show that a single pair of ReLU nets can achieve a notion of limited-capacity-equilibrium. 
The intuition underlying our theorem is as follows: neural nets have a particular structure that interleaves matrix multiplications and simple non-linearities. The matrix multiplications in one layer of a neural net compute linear combinations of functions encoded by the other layers. In other words, neural nets are (non-)linear mixtures of their sub-networks. 

Our main result can be related to games with bounded rationality~\citep{simon1969sciences}: when the players pick parametrized models, they are limited by the representational power of the class of models accessible (e.g., a fixed architecture).
It is instructive to discuss the relative merits of that limited-capacity aspect occuring. 
On the one hand, if one had access to any function/distribution an infinite-capacity equilibrium would exist (because the whole function/distribution space is convex). However, this quantity may not be realistic, e.g., in GANs, the optimal infinite-capacity generator must represent the distribution of `real-world' images. If such a concept is not tractable, it seems unrealistic to expect limited capacity agents, such as humans or computers, to find it~\citep{papadimitriou2007complexity}. 
On the other hand, our work shows that one can efficiently approximate some equilibria when working with neural networks. In a similar vein as the games with bounded rationality, these equilibria capture the notion that agents--and humans--that play complex games have a limited capacity. It seems to be a more reasonable concept to consider the optimal way to play complex games such as Poker of StarCraft II that are multi-step with imperfect information. 

\acknowledgments{Gauthier Gidel would like to thank Ian Gemp for the helpful discussions.}

\bibliographystyle{abbrvnat}
\bibliography{references.bib}
\appendix

\onecolumn
\aistatstitle{A Minimax Theorem for Nonconcave-Nonconvex Games: \\
Supplementary Materials}

\section{Interpretation of Equilibria in Latent Games}
\label{app:interpretation_equilibria}

In latent games, players embed in mapping spaces in order to solve the game. 
When we consider a standard normal form game $\varphi$ that we try to solve using mappings to approximate mixtures of strategies, we are actually playing a limited capacity version of the game that heavily depends on the expressivity of the mappings in the classes $\mathcal{F}$ and $\mathcal{G}$. 

Such a limitation may be interpreted as limitations on the skills of the players. It intuitively makes sense that such limitations would change the optimal way to play the game: the optimal way to play StarCraft II is different for players that can perform 10 versus 100 actions per second. Thus,
if the goal is to train agents to compete with humans, one needs to set a class $\mathcal{G}$ that (roughly) corresponds to human skills. Setting ``fair'' constraints on the RL agents trained to play the game of StarCraft II has been an important issue~\citet{vinyals2019grandmaster} and can be understood as setting the right class $\mathcal{G}$ in a latent game.

Similarly a player would not play poker the same way if they had no memory of their opponents' behavior in previous games. 

Similarly, in the context of Generative Adversarial Networks, it has been argued that setting a restricted function class for the discriminator provides a more meaningful loss and describes an achievable learning task for the generator~\citet{arora2017generalization, huang2017parametric}. The final task is to generate pictures that are realistic according to the human metric. Such task is way looser -- and thus easier to achieve -- than for instance minimizing the KL divergence or the Wasserstein distance between the real data distribution and the generated distribution. 

To sum-up, the equilibrium of a latent game provides a notion of limited-capacity-equilibrium that can define a target that correspond to agents with expressive and realistic behavior. In many tasks, our goal is to train agents that outperform human using human realistic limitations: 
it is important to constrain the agent in order to prevent it to play $10^5$ actions per minute but it is also important to constrain its opponent because we would like opponent to try to exploit the main agent in a semantically meaningful way and not by designing very specific 'adversarial example' strategies --e.g., very precise positions of units that breaks the vision system of the main agent --  that a human player could not perform.

This idea of modeling the limitations of realistic players play suboptimally is related to the notion of games with bounded rationality~\citep{simon1969sciences,rubinstein1998modeling,papadimitriou1994complexity,kalai1990bounded} or bounded computation~\citep{halpern2015algorithmic}. However, bounded rationality models players that do not optimize their reward function~\citep{rubinstein1998modeling}, the corresponding literature aims to model a process a choice for players not always maximizing their reward. Bounded computation refers to studies of games where players pay for the (time) complexity of the strategy they use. The notion of limited-capacity in latent games is a limitation on the representative power of the function (or distribution) spaces. The literature has not thoroughly considered \emph{limitations on representational power} -- a gap that is critical to address, given that neural nets are now a major workhorse in AI and ML.

\section{Relevance of the Minimax theorem in the Context of Machine Learning}{
\label{app:importance of minimax}
}

A notorious ML application which has a minimax formulation is \emph{adversarial training} where a classifier is trained to be robust against adversarial attack.
From a game-theoretic perspective, the adversarial attack is picked after the classifier $f$ is set and thus it corresponds to a best response. 
From a learning perspective, the goal is to learn to be robust to adversarial attacks \emph{specifically designed} against the current classifier. Such an equilibrium is called a Stackelberg Equilibrium~\citep{conitzer2006computing}.


In games with imperfect information such as Colonel Blotto, Poker, or StarCraft II the players must commit to a strategy without the knowledge of the strategy picked by their opponent. In that case, the agents cannot design attacks specific to their opponent, because such attacks may be exploitable strategies. It is thus strictly equivalent to consider that the players simultaneously pick their respective strategies and then reveal them. Thus, a meaningful notion of playing the game must have a value and an equilibrium. 

In machine learning applications, each player is trained using local information (though gradient or RL based methods). Because the behavior of the players changes slowly, they cannot have access to the best response against their opponent. In order to illustrate that point, let us consider the example of Generative Adversarial Networks. The two agents (the generator and the discriminator) are usually sequentially updated using a gradient method with similar step-sizes. During training, one cannot expect an agent to find a best response in a single (or few) gradient steps. 
To sum-up, since local updates are performed one must expect to reach a point (if it exists) that is locally stable. In this work, we show that there actually exists a \emph{global} approximate equilibrium for a large class of parametrized games.

\section{Proof of results from Section~\ref{sec:minimax_latent}}
\label{app:proof}

\subsection{Proof of Proposition~\ref{prop:minimax_hull}}

Before proving this proposition let us state Sion's minimax theorem.

\begin{theorem}[Minimax theorem~\citep{sion1958general}]\label{thm:sion's theorem}
If $U$ and $V$ are convex and compact sets and if the sublevel sets of $\varphi(\cdot,v)$ and $-\varphi(u,\cdot)$ are convex then,
\begin{equation}
    \max_{u \in U} \min_{v \in V} \varphi(u,v) = \min_{u  \in U} \max_{v \in V} \varphi(u,v)
\end{equation}
\end{theorem}

Let us now state our proposition.

\minimaxHull*

\begin{proof}
For simplicity and conciseness we note, $\gF = \gF_\Omega$ and $\gG= \gG_\Theta$.
     The sets $\hull(\gF)$ and $\hull(\gG)$ are convex by construction. However, they are not compact in general. However, since $\gG$ is assumed to be a compact set we then have that under mild assumptions (namely, that $\gF$ and $\gG$ belong to a completely metrizable locally convex space) that the closure of $\hull(\gG)$ is compact~\citep[Theorem 5.20]{aliprantisinfinite}. Thus, we can apply Sion's theorem to get, 
     \begin{equation}
    \min_{p \in \closure(\hull(\gG))} \max_{f \in \closure(\hull(\gF))} \tilde \varphi(f,p) = \max_{f \in \closure(\hull(\gF))} \min_{p \in \closure(\hull(\gG))} \tilde \varphi(f,p)
\end{equation}
\end{proof}
Moreover there exists $(w_i)_{i\geq 0}$, $(\theta_i)_{i\geq 0}$, $\lambda_i \geq 0\,,\, \sum_{i\geq 0} \lambda_i =1$ and $\rho_i \geq 0\,,\, \sum_{i\geq 0} \rho_i =1$ such that,
\begin{equation}
    V(\Omega,\Theta) = \tilde \varphi \Big( \sum_{i\geq 0} \lambda_i f_{w_i},  \sum_{i\geq 0} \rho_i p_{\theta_i} \Big)\,.
\end{equation}
This comes from the fact that any element in $\closure(\hull(\gF))$ can be written as $\sum_{i\geq 0} \lambda_i f_{w_i}$:
 \begin{lemma}\label{lemma:compact_convex_closure}Let $U$ be a compact set that belongs to a completely metrizable locally convex space. Then the closure of the convex hull of $U$ is compact and we have that 
 $\closure(\hull(U)) = \{\sum_{i\geq 0} \lambda_i u_i \,,\, \lambda_i \geq 0\,, \, \sum_{i\geq 0} \lambda_i =1 \,,\, u_i \in U\}$.
 \end{lemma}
\begin{proof}
Let us consider a sequence $(u_n) \in conv(U)^{\sN}$, we have $u_n = \sum_{i=0}^{K_n} \lambda_{i,n} u_{i,n}$ where $u_{i,n} \in U\,,\; \forall i,n \in \sN $. Since $\lambda_{i,n} \in [0,1]$ and $u_{i,n} \in U$ that are compact sets these sequences have a convergent subsequence.
By Cantor diagonalization process, $(x_n)$ has a convergent subsequence.
\end{proof}



\subsection{Proof of Theorem~\ref{thm:finiteKepsilon}}


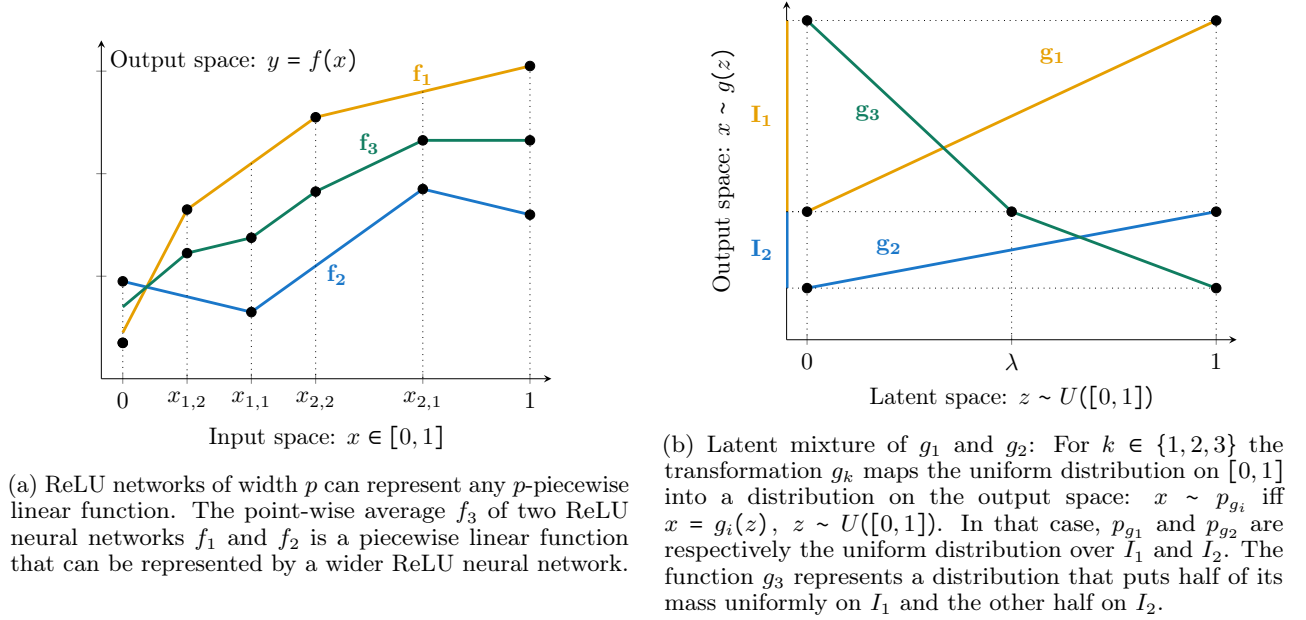
\begin{figure*}
\centering
\begin{subfigure}{.48\linewidth}
    \centering
    \scalebox{.9}{
    \begin{tikzpicture}
\tikzmath{
\r1 = 8;
\r2 = 3;
\r3 = 1;
\b1 = 1;
\b2 = 21;
\b3 = 41;
\b4 = 20;
\b5 = -8;
\b6 = 52;
\q1 = 5;
\q2 = 10;
\q3 = 20;
\x1 = 4; \y1 = \r1*\x1 + \b1; \z1 = .5*(\r1-\r3)*\x1 + .5*(\b1+\b4);
\x2 = 7; \y2 =-\r3 * \x2 + \b4; \z2 = .5*(\r2-\r3)*\x2 + .5*(\b2+\b4);
\x3 = 10; \y3 = \r2 * \x3 + \b2; \z3 = .5*(\r2+\r2)*\x3 + .5*(\b2+\b5);
\x4 = 15; \x5 = 20; \z4 = .5*(\r3+\r2)*\x4 + .5*(\b3+\b5); \z5 =.5*(\b3+\b6);
\y4 = \r2*\x4 + \b5; \y5 = \r3*\x5 + \b3; \y6 = -\r3*\x5 + \b6;
 } 
\begin{axis}[
axis x line=middle,
axis y line=middle,
ylabel=Output space: ${y=f(x)}$,
xlabel=Input space: ${x \in [0,1]}$,
xtick={1,\x1,\x2,\x3, \x4, \x5},
xticklabels={$0$,$x_{1,2}$,  $x_{1,1}$, $x_{2,2}$, $x_{2,1}$, $1$},
xlabel near ticks,
yticklabels={,,}
ylabel near ticks,
xmax=\x5+1,
ymax=\y5+5,
xmin=0,
ymin=0,
width=\textwidth,
height=.8\textwidth
]
\addplot[domain=1:\x1, draw=orangeblind, very thick] {\r1*x + \b1};
\addplot[domain=\x1:\x3, draw=orangeblind, very thick] {\r2*x + \b2};
\addplot[domain=\x3:\x5, draw=orangeblind, very thick] {\r3*x + \b3};
\node[anchor=north] (f1) at (axis cs: {.5*(\q2+\q3)}, {\r3*.5*(\q2+\q3)+ \b3+7}) {${\color{orangeblind}\mathbf{f_1}}$};
\addplot[domain=1:\x2, draw=blueblind, very thick] {-\r3*x + \b4};
\addplot[domain=\x2:\x4, draw=blueblind,very thick] {\r2*x + \b5};
\addplot[domain=\x4:\x5, draw=blueblind, very thick] {-\r3*x + \b6};
\node[anchor=north] (f1) at (axis cs: {.5*(\x2+\x4)}, {\r2*.5*(\x2+\x4)+ \b5-1}) {${\color{blueblind}\mathbf{f_2}}$};
\addplot[domain=1:\x1, draw=greenblind, very thick] {.5*(\r1-\r3)*x + .5*(\b1+\b4)};
\addplot[domain=\x1:\x2, draw=greenblind, very thick] {.5*(\r2-\r3)*x + .5*(\b2+\b4)};
\addplot[domain=\x2:\x3, draw=greenblind, very thick] {.5*(\r2+\r2)*x + .5*(\b2+\b5)};
\addplot[domain=\x3:\x4, draw=greenblind, very thick] {.5*(\r3+\r2)*x + .5*(\b3+\b5)};
\addplot[domain=\x4:\x5, draw=greenblind, very thick] {.5*(\b3+\b6)};
\node[anchor=north] (f1) at (axis cs: {.5*(\x3+\x4)}, {.25*(\r3+\r2)*(\x3+\x4) + .5*(\b3+\b5) + 8}) {${\color{greenblind}\mathbf{f_3}}$};
\draw[dotted] (axis cs:\x1,\y1) -- (axis cs:\x1, 0);
\draw[dotted] (axis cs:\x2,\r2*\x2 + \b2) -- (axis cs:\x2, 0);
\draw[dotted] (axis cs:\x3,\y3) -- (axis cs:\x3, 0);
\draw[dotted] (axis cs:1,19) -- (axis cs:1, 0);
\draw[dotted] (axis cs:\x4,\r3*\x4 + \b3) -- (axis cs:\x4, 0);
\draw[dotted] (axis cs:\x5,\y5) -- (axis cs:\x5, 0);
\addplot[only marks,mark=*] coordinates{(1,7)(\x1,\y1)(\x2,\y2)(\x3,\y3)(1,19)(\x4,\y4)(\x5,\y5)(\x5,\y6)};
\addplot[only marks,mark=*] coordinates{(1,7)(\x1,\z1)(\x2,\z2)(\x3,\z3)(\x4,\z4)(\x5,\z5)};

\end{axis}
\end{tikzpicture}
    }
    \caption{ReLU networks of width $p$ can represent any $p$-piecewise linear function.
    The point-wise average $f_3$ of two ReLU neural networks $f_1$ and $f_2$ is a piecewise linear function that can be represented by a wider ReLU neural network.
    }
    \label{fig:convex_function}
\end{subfigure}
\quad
\begin{subfigure}{.48\linewidth}
    \centering
    \scalebox{.9}{
    \begin{tikzpicture}

\tikzmath{
\r1 = 8;
\r2 = 3;
\r3 = 1;
\b1 = 1;
\b2 = 21;
\b3 = 12.4;
\b4 = 20;
\b5 = -8;
\b6 = 52;
\q1 = 5;
\q2 = 10;
\q3 = 20;
\x1 = 4; \y1 = 5.4; \z1 = .5*(\r1-\r3)*\x1 + .5*(\b1+\b4);
\x2 = 7; \y2 = 8.4+5; \x2 + \b4; \z2 = .5*(\r2-\r3)*\x2 + .5*(\b2+\b4);
\x3 = 11; \y3 = \b3+1; \z3 = .5*(\r2+\r2)*\x3 + .5*(\b2+\b5);
\x4 = 15; \x5 = 21; \z4 = .5*(\r3+\r2)*\x4 + .5*(\b3+\b5); \z5 =.5*(\b3+\b6);
\y4 = \b3+21; \y5 = \r3*\x5 + \b3; \y6 = -\r3*\x5 + \b6;
\I2 = .5*(\y1+\y2); \I1 = .5*(\y3+\y4);
 } 
\begin{axis}[
axis x line=middle,
axis y line=middle,
width=\textwidth,
height=.8\textwidth,
ylabel=Output space: ${\normalsize x \sim g(z)}$,
xlabel=Latent space: ${\normalsize z \sim U([0,1])}$,
xtick={1,\x3, \x5},
xticklabels={$0$, $\lambda$, $1$},
xlabel near ticks,
ytick={0, \I2, \I1},
ytick style={draw=none},
yticklabels={0,${\color{blueblind}\mathbf{I_2}}$, ${\color{orangeblind}\mathbf{I_1}}$},
ylabel near ticks,
xmax=\x5+1,
ymax=\x5 +\b3+2,
xmin=0,
ymin=0
]
\addplot[domain=1:\x5, draw=orangeblind, very thick] {x+\b3};
\node[anchor=north] (f1) at (axis cs: {\x3+2}, { \b3+\x3 + 8}) {${\color{orangeblind}\mathbf{g_1}}$};

\addplot[domain=1:\x5, draw=blueblind, very thick] {.4*x + 5};
\node[anchor=north] (f2) at (axis cs: {5}, {.4*\x5+3}) {${\color{blueblind}\mathbf{g_2}}$};

\addplot[domain=1:\x3, draw=greenblind, very thick] {2* (12 - x) + \b3-1};
\addplot[domain=\x3:\x5, draw=greenblind, very thick] {.8*(32-x) - 3.4};
\node[anchor=north] (f1) at (axis cs: {4}, {\b3+13}) {${\color{greenblind}\mathbf{g_3}}$};
\draw[dotted] (axis cs:11,0) -- (axis cs:11, \y2);
\draw[dotted] (axis cs:1,\y4) -- (axis cs:1, 0);
\draw[dotted] (axis cs:\x5,\y4) -- (axis cs:\x5, 0);
\draw[blueblind, ultra thick] (axis cs:0,\y1) -- (axis cs:0, \y2);
\draw[orangeblind, ultra thick] (axis cs:0,\y3) -- (axis cs:0, \y4);

\draw[dotted] (axis cs:0,\y2) -- (axis cs:\x5, \y2);
\draw[dotted] (axis cs:0,\y1) -- (axis cs:\x5, \y1);
\draw[dotted] (axis cs:\x5,\x5+\b3) -- (axis cs:0, \x5+\b3);
\draw[dotted] (axis cs:0,\b3+1) -- (axis cs: 1, \b3+1);
\addplot[only marks,mark=*] coordinates{(1,\y1)(\x5,\y1)(\x3,\y2)(21,13.4)(1,\b3+1)(1,\b3+21)(21,\b3+21)};

\end{axis}
\end{tikzpicture}
    }
    \caption{
    Latent mixture of $g_1$ and $g_2$: For $k \in \{1,2,3\}$ the transformation $g_k$ maps the uniform distribution on $[0,1]$ into a distribution on the output space: $x \sim p_{g_i}$ iff $x = g_i(z)\,, \; z \sim U([0,1])$. 
    In that case, $p_{g_1}$ and $p_{g_2}$ are respectively the uniform distribution over $I_1$ and $I_2$. The function $g_3$ represents a distribution that puts half of its mass uniformly on $I_1$ and the other half on $I_2$.
    }
    \label{fig:preMappingAvg}
\end{subfigure}
\caption{Difference between pointwise averaging of function and latent mixture of mapping.}
\label{fig:avgvsmixture}
\end{figure*}
We will prove a result a bit more general that the result stated in the main paper, 
\begin{restatable}{theorem}{FiniteKEpsilon}
Let $\varphi$ a game that satisfies the assumptions of Proposition~\ref{prop:minimax_hull}. 
If $\tilde \varphi$ is bilinear and $\varphi$ is $L$-Lipschitz then, 
\begin{equation}
    K_\epsilon^\Omega  \leq \frac{4D_w}{\epsilon^2}\ln(\mathcal{N}(\Theta,\tfrac{\epsilon}{2L})) 
    \quad \text{and} \quad 
    K_\epsilon^\Theta  \leq \frac{4D_\theta}{\epsilon^2}\ln(\mathcal{N}(\Omega,\tfrac{\epsilon}{2L})) 
\end{equation}
where $\mathcal{N}(\gH,\epsilon')$ is the number of $\epsilon'$-balls necessary to cover the set $A$ and the quantities $D_w$ and $D_\theta$ are defined as
$D_w:= \max_{w,w',\theta} \varphi(w,\theta) - \varphi(w',\theta)$ and $D_\theta := \max_{w,\theta, \theta'} \varphi(w,\theta) - \varphi(w,\theta') $.
\end{restatable}
In the literature,the quantity $\mathcal{N}(\gH,\epsilon')$ is called covering number of the set $\gH$. By definition of compactness, it is finite when $\gH$ is compact. It is a complexity measure of the set $\gH$ that has been extensively studied in the context of generalization bounds~\citep{mohri2012foundations,shalev2014understanding}.

\begin{proof}
This proof is largely inspired from the proof of~\citep[Theorem 2]{lipton1994simple} and~\citep[Theorem B.3]{arora2017generalization}. The difference with~\citep[Theorem B.3]{arora2017generalization} is that we make appear a notion of condition number and we provide this proof in a context more general than~\citep[Theorem B.3]{arora2017generalization} who was focusing on GANs.

One way to insure that $D_w$ amd $D_\theta$ have bounded value is by 
assuming that $\Theta$ and $\Omega$ have a finite diameter, we then have that the values of $D_w$ and $D_\theta$ are respectively bounded by $L \diam(\Theta)$ and $L \diam(\Omega)$. Note that is practice one also may have that the payoff is bounded between $-1$ (losing) and $1$ (winning).   

By Proposition~\ref{prop:minimax_hull} we have that there exists $f^*$ and $p^*$ such  
     \begin{equation}
        V(\Omega,\Theta) = \tilde \varphi \Big( f^*,p^* \Big)\,,
\end{equation}
where,

\begin{equation}\label{eq:f_k}
     f^* := \sum_{k=1}^\infty \lambda_k f_{w_k}  \;\text{and} \;
     p^* := \sum_{k=1}^\infty \rho_k p_{\theta_k} 
\end{equation}
where $w_k, \theta_k \in \Omega \times \Theta\,,\, \rho_k,\lambda_k > 0 \,,\, \sum_{k=1}^\infty \lambda_k = \sum_{k=1}^\infty \rho_k = 1$.

Now let us consider the mixture
\begin{equation}
    f_n^* := \frac{1}{n}\sum_{k=1}^n f_{w_k}
\end{equation}
where $w_k, 1\leq k\leq n$ are defined in~\eqref{eq:f_k} and are drawn independently from the multinomial of weights $(\lambda_k)_{k=1\ldots \infty}$.  

Let assume that $\Theta$ has a finite covering number $\mathcal{N}(\Theta,\epsilon/(2L))$ (in the following we will show that if $\Omega$ is compact then, its covering number is finite and we will give an explicit bound on it when $\Omega \subset \R^p$).
Let us recall that covering number of $\Theta$ is the smallest number of $\epsilon$ balls needed to cover $\Theta$.
Let us consider $\theta_i\,,\, 1\leq i \leq \mathcal{N}(\Theta,\tfrac{\epsilon}{2 L})$ the center of these balls where $L$ is the Lipschitz constant of $\varphi$. 
Using Hoeffding's inequality, for any $\theta_i\,,\, 1\leq i \leq \mathcal{N}(\Theta,\tfrac{\epsilon}{2 L})$ we have that,
\begin{equation}
    \sP( \tilde \varphi(f_n^*,p_{\theta_i}) - \tilde \varphi(f^*,p_{\theta_i}) < \epsilon/2) \leq e^{\frac{-n \epsilon^2}{2D_w^2}}
\end{equation}
where $D_w$ is a bound on the variations of $\varphi$ defined as 
\begin{equation}
    D_w:= \max_{w,w',\theta} \varphi(w,\theta) - \varphi(w',\theta) \,.
\end{equation}
Note that because we assumed that $\tilde \varphi$ is bilinear, the bound on the variations of $\varphi$ is also valid for the variations of $\tilde \varphi$. More precisely, we have
\begin{align}
    \tilde \varphi(\sum_i \lambda_i f_{w_i},p_{\theta}) - \tilde \varphi(\sum_i \lambda_i' f_{w_i'} ,p_{\theta}) 
    &= \sum_i \lambda_i( \varphi( w_i,\theta) - \varphi(w'_i ,\theta)) \\
    &\leq \sum_i \lambda_i D_w = D_w.
\end{align}
Thus, using standard union bounds, 
\begin{align}
    &\sP \big( \tilde\varphi(f_n^*,p_{\theta_i}) - \tilde \varphi(f^*,p_{\theta_i}) < \epsilon/2 \,, \, \forall 1\leq i \leq  \mathcal{N}(\mathcal{G},\tfrac{\epsilon}{2 \tilde L} ) \big)  \notag \\
    & \quad \leq  \mathcal{N}(\mathcal{G},\tfrac{\epsilon}{2\tilde L}) e^{\frac{-n \epsilon^2}{2D_w^2}}
\end{align}
Let us now consider 
\begin{equation}
    \hat p_n \in \arg \min_{p  \in \closure(\hull(\gG))} \tilde \varphi(f_n^*,p) = \arg \min_{\theta  \in \Theta} \tilde \varphi(f_n^*,p_\theta)
\end{equation}
Note that this minimum is achieved with $q \in \gG$ because we assumed that the function $\tilde \varphi$ is bilinear (and thus a minimum with respect to a convex hull is always achieved at an atom).\footnote{Note that we could get rid of the bilinear assumption by replacing the covering number of $\Theta$ by the covering number of $\hull(\gG)$. However the asymptotic behavior of the latter (when $\epsilon \to 0$) may be challenging to bound. We thus decided to focus on bilinear examples since the covering number for finite dimensional compact sets is a well studied quantity.} Thus there exists $\hat \theta_n \in \Theta$ such that $\hat p_n = p_{\hat \theta_n}$.

Since $\varphi$ is $L$-Lipschitz and since we have that $(\theta_i)$ is an $\tfrac{\epsilon}{2L}$-covering there exists a $\theta_i$ that is $\tfrac{\epsilon}{2L}$-close to $\hat \theta_n$ and thus,
\begin{equation}
     |\tilde \varphi(f_n^*,p_{\theta_i})- \min_{p \in \closure(\hull(\gG))} \tilde \varphi(f_n^*,p)| = \Big| \sum_{k=1}^n \tfrac{1}{n} ( \varphi(w_k,\theta_i) - \varphi(w_k,\hat \theta_n)) \Big|
     \leq  \epsilon/2
     \,.
\end{equation}
When we have that $\tilde \varphi(f^*,p_{\theta_i})- \tilde\varphi(f_n^*,p_{\theta_i}) < \epsilon/2$ (which is true with high probability) we have,
\begin{align}
    \min_{p \in \closure(\hull(\gG))}\varphi(f_n^*,p) 
    & \geq  \tilde \varphi(f_n^*,p_{\theta_i}) - \epsilon/2 \\
    & > \tilde \varphi(f^*,p_{\theta_i}) - \epsilon \\
    & =
    V(\Omega,\Theta) - \epsilon
\end{align}
Thus for $n > \frac{4D_w^2}{\epsilon^2}\ln(\mathcal{N}(\Theta,\tfrac{\epsilon}{2 L}))$ we have,
\begin{equation}
    \sP\big(\min_{p \in \closure(\hull(\gG))}\varphi(f_n^*,p) > V(\Omega,\Theta) - \epsilon \big) < 1
\end{equation}

Since this probability is strictly smaller than one, for any $\epsilon' > 0$, among all the possible sampled $f_n^*$ there exist at least one such that 
\begin{equation}
     \min_{p \in \closure(\hull(\gG))}\varphi(f_n^*,p) > \Val_L - \epsilon  \,.
\end{equation}
Thus, 
\begin{equation}
    K_\epsilon^{\Omega} \leq  \frac{4D_w^2}{\epsilon^2}\ln(\Theta,\tfrac{\epsilon}{2L})\,.
\end{equation}
A similarly we can prove a bound on $K_\epsilon^{\Theta}$.

\end{proof}
Then, we will use a simple bound for the covering number $\Theta \subset \R^d$ that can be found in~\citet{shalev2014understanding}, 
\begin{equation}
    \log \mathcal{N}(\Theta,\tfrac{\epsilon}{2L}) \leq d \log( \tfrac{4 L R \sqrt{d}}{\epsilon})\,.
\end{equation}
that leads to
\begin{equation}
    K_\epsilon^\Omega  \leq \frac{4D_w^2 d}{\epsilon^2} \log( \frac{4 L R \sqrt{d}}{\epsilon})
\end{equation}

\subsection{Proof of Proposition~\ref{prop:conv_function}}
\convexNetsFunction*
\begin{proof}
We will prove this result for an arbitrary convex combination. 
Let us start with the proof for a two-layers neural network of width $W$. It can be written as
\begin{equation}
    g(x) = \sum_{i=1}^W  a_i \sigma(c_i^\top x + d_i) + b_i 
\end{equation}
where $a_i,b_i \in \R^{d_{out}}, \, c_i \in \R^{d}$, and $d_i \in \R$ and $\sigma$ is any given non-linearity. Then, let us consider $K$ such functions with $p$ parameters, then any convex combination of these $K$ functions can be written as,
\begin{equation}
    f(x) = \sum_{k=1}^K\sum_{i=1}^{W_k}  \lambda_k(a_{i,k} \sigma(c_{i,k}^\top x + d_{i,k}) + b_{i,k})
\end{equation}
where $\lambda_k \geq 0\,,\, 1\leq k \leq K$ and $\sum_{k=1}^K \lambda_k = 1$.

Setting $\tilde a_{i,k} := \lambda_ka_{i,k} $ and $\tilde b_{i,k} := \lambda_k b_{i,k}$, we have that 
\begin{equation}
    f(x) = \sum_{(i,k)}  \tilde a_{i,k} \sigma(c_{i,k}^\top x + d_{i,k}) + \tilde b_{i,k}
\end{equation}
which is a neural network using the non-linearity $\sigma$ with $K\cdot p$ parameters.

Let us now consider $K$ neural networks $(f_{w_k})_{k \in 1\ldots K}$ with $p$ parameters $w_k \in [-R,R]^p$ and the same architecture, and $\lambda_k \geq 0\,,\, 1\leq k \leq K$ such taht $\sum_{k=1}^K \lambda_k = 1$. By multiplying the last layer of $f_k$ by $\lambda_k$ and concatenating each layer to get $w \in [-R,R]^{p \cdot K}$ we have that
\begin{equation}
    \sum_{k=1}^K \lambda_k f_{w_k} = f_w
\end{equation}
is a neural network with $K\cdot p$ parameters.
\end{proof}

\subsection{Proof of Proposition~\ref{prop:conv_distrib}}

\convexNetsDistribution*
\begin{proof}

We will prove the first part of this theorem for an arbitrary number $K$ of mappings.
Let $g$ be a two-layers ReLU network of width $p$,
the probability distribution $\pi_g$ induced by $g$ verifies,
\begin{equation}
    \pi_g(S) = \ell(g^{-1}(S)) \,,\quad \forall S \text{ measurable in } [0,1]^{d_{out}}\,.
\end{equation}
where $\ell$ is the Lesbegues measure on $[0,1]^d$. 
The convex combination $\pi$ of $\pi_{g_1}, \ldots, \pi_{g_K}$ verifies,
\begin{align*}
    \pi (S) 
    &:= \Big(\sum_{k=1}^K \lambda_k\pi_{g_k}\Big)(S)  \\
    &= \sum_{\substack{1\leq k \leq K}} \lambda_k \ell(g_k^{-1}(S))
\end{align*}
Using the properties of the Lesbegues measure we have that $\forall \lambda>0, b\in \R$
\begin{equation}
 \lambda \ell(U) = \ell(\lambda U + b)  \,,
\end{equation}
where $\lambda U$ is the dilation of the set $U$  and $U+b$ its translation by $b$. thus, we have that for any $b_k \in \R \,,\, k=1\ldots K$,
\begin{align*}
    \pi (S) 
    &= \sum_{\substack{1\leq k \leq K}}  \ell(\lambda_kg_k^{-1}(S) + b_k)
\end{align*}
Now notice that,
\begin{align}
    \lambda_kg_k^{-1}(S) + b_k 
    &= \{ \lambda_k x +b_k\, :\, x\in [0,1] \,,\, g_k(x) \in S \} \\
    &=  \{ x\, :\, x\in [b_k,\lambda_k + b_k] \,,\, g_k(x/\lambda_k +b_k) \in S \} 
\end{align}
Then, setting $b_k := \sum_{i=0}^{k-1} \lambda_k \in [0,1]$, we get by construction that $b_{k+1} = b_k + \lambda_k$ and thus that the sets $S_k:=[b_k,\lambda_k + b_k]$ are a partition of $[0,1]$.
Finally, if we note $\tilde g$ the function, 
\begin{equation}
    \tilde g(x) = g_k(x/\lambda_k +b_k) \; \text{if} \; x \in [b_k, b_{k+1}] 
\end{equation}
We have by construction (and by the fact that $S_k$ are disjoints)
\begin{equation}
     \pi_{\tilde g}(S) = \sum_{\substack{1\leq k \leq K}} \ell(\lambda_kg_k^{-1}(S) + b_k) = \Big(\sum_{k=1}^K \lambda_k\pi_{g_k}\Big)(S)
\end{equation}
However, the proof is not over because $\tilde g$ may not be continuous it cannot correspond in general to a neural network. We will now construct a neural network that approximate the distribution induced by $\tilde g$.
Let us introduce the approximated "step" function $h_\delta$ that is a ReLU net with $3$ parameters and 
\begin{equation}
    h_\delta(x) :=  \frac{1}{\delta}\lfloor x \rfloor_+ -  \frac{1}{\delta}\lfloor x -\delta  \rfloor_+ = \left\{ \begin{aligned} &0 \;\text{if } x <0 \\ & 1 \;\text{if } x > \delta \\ & x/\delta \;\text{otherwise.} 
    \end{aligned} \right.
\end{equation}
Thus we can introduce the
ReLU net $\tilde g_k$ defined as 
\begin{align}
    \tilde g_k (x) 
    &:=  g_k(x/\lambda_k +b_k)  - g_k(0) h_\delta(-x + b_k) - g_k(1) h_\delta(x + b_{k+1} )\\ 
    &= \left\{ \begin{aligned} &0 \;\text{if } x <b_k \text{ or } x>b_{k+1} \\ &g_k(x/\lambda_k +b_k) \;\text{if } b_k +\delta <x < b_{k+1} -\delta 
    \end{aligned} \right.
\end{align}
The second line is due to the fact that we assumed that $g_k(x) = g_k(0)\,,\, \forall x <0$ and $g_k(x) = g_k(1)\,,\, \forall x >1$.
Finally we have that the sum of $\tilde g_k$ for $k=1,\ldots K$ is a ReLU neural network with $K(p+6)$ parameters such that
\begin{align*}
    TV(\pi, \pi_{\sum_k \tilde g_k}) &= \sup_{S} |\pi(S) - \pi_{\sum_k \tilde g_k}(S)| \\
    &\leq K\delta  
\end{align*} Moreover, since $g_k$ has $p$ parameters in $[-R,R]$ we have that $\tilde g_k$ has $p+6$ parameters that are in $[-R/\lambda_k,R/\lambda_k]$. 
Since we assumed that the parameters of the ReLU network should be bounded by $KR$ we have that we cannot pick the parameters $g_k(0)/\delta$ and $g_k(1)/\delta$ larger than $KR$.

Thus by setting $\lambda_k = 1/K$, there exists a ReLU network with $K(p+6)$ parameters in $[-KR,KR]$ such that,
\begin{align*}
    TV(\pi, \pi_{\sum_k \tilde g_k}) &= \sup_{S} |\pi(S) - \pi_{\sum_k \tilde g_k}(S)| \\
    &\leq \frac{1}{R}  
\end{align*}

\end{proof}

\subsection{Proof of Theorem~\ref{thm:non-conv_minimax}}
\label{app:minimax}

\begin{reptheorem}{thm:non-conv_minimax} Let $\varphi$ be a $L$-Lipschitz nonconcave-nonconvex game with values bounded by $D$ that follows Assumption~\ref{assump:nonconvex-nonconcave_convex} for which the payoff $\tilde \varphi$ is bilinear and $\tilde L$ Lipschiz. The players are assumed to be parametrized neural networks $g:\R^d\to \R^{d_{out}}$ with $p$ parameters smaller than $R$, and satisfies one of the three following cases:
\begin{itemize}
    \item Both players pick functions. 
    For any $\epsilon >0$ there exists $(w_\epsilon^*,\theta_\epsilon^*) \in [-R,R]^{2p}$ s.t.,
\begin{equation} 
    \min_{\substack{\theta \in \R^{p_\epsilon} \\ \|\theta\|\leq R}} \varphi(w_{\epsilon}^*,\theta) + \epsilon \geq \max_{\substack{w \in \R^{p_\epsilon}\\ \|w\| \leq R}} \varphi(w,\theta_\epsilon^*)
    \,.
\end{equation}where $p_\epsilon \geq \frac{\epsilon}{2D}\sqrt{\frac{p}{\log(4L\sqrt{p}/\epsilon)}}$.
    \item The first player picks distributions whose generating function is a ReLU nets with $d_{in}=1$. The second player picks functions. This is for instance the setting of WGAN (Example~\ref{example:WGAN}). For any $\epsilon >0$ there exists $(w_\epsilon^*,\theta_\epsilon^*) \in [-R,R]^{2p}$ s.t.,
    \begin{equation} 
    \min_{\substack{\theta \in \R^{p_\epsilon} \\ \|\theta\|\leq R_\epsilon}} \varphi(w_{\epsilon}^*,\theta) + \epsilon + \frac{\tilde L}{R} \geq \max_{\substack{w \in \R^{p_\epsilon}\\ \|w\| \leq R_\epsilon}} \varphi(w,\theta_\epsilon^*) 
    \,.
\end{equation}
where $p_\epsilon \geq \frac{\epsilon}{2D}\sqrt{\frac{p}{\log(4L\sqrt{p}/\epsilon)}} - 6 \,,\,
    R_\epsilon \geq R\frac{p_\epsilon}{p}$, and  the subnetworks generating the distributions (see Eq.~\ref{eq:induced_distribution}) takes their values in $[0,1]^d$ and are constant outside of $[0,1]$.
    \item Both players pick distributions whose generating function is a ReLU nets with $d_{in}=1$. This is for instance the setting of the Blotto game (Examples~\ref{example:Blotto}). For any $\epsilon >0$ there exists $(w_\epsilon^*,\theta_\epsilon^*) \in [-R,R]^{2p}$ s.t.,
  \begin{equation}
    \min_{\substack{\theta \in \R^{p_\epsilon} \\ \|\theta\|\leq R_\epsilon}} \varphi(w_{\epsilon}^*,\theta) + \epsilon + \frac{2\tilde L}{R} \geq \max_{\substack{w \in \R^{p_\epsilon}\\ \|w\| \leq R_\epsilon}} \varphi(w,\theta_\epsilon^*) 
    \,.
\end{equation}
where $p_\epsilon \geq \frac{\epsilon}{2D}\sqrt{\frac{p}{\log(4L\sqrt{p}/\epsilon)}} - 6  \,,\,
    R_\epsilon \geq R\frac{p_\epsilon}{p}$ and the subnetworks generating the distributions (see Eq.~\ref{eq:induced_distribution}) takes their values in $[0,1]^d$ and are constant outside of $[0,1]$.
\end{itemize} 
\end{reptheorem}
\begin{proof}
Let $\epsilon>0$ and let us consider ReLU networks with $p_\epsilon$ parameters in $[-R_\epsilon,R_\epsilon]$ (we will set those quantities later). For simplicity here $\Omega = \Theta = [-R_\epsilon,R_\epsilon]^p$.
Theorem~\ref{thm:finiteKepsilon} says that an $\epsilon$-equilibrium can be achieved with a uniform convex combination of $K_\epsilon$ networks.

Let us consider the case where the first player is a function and the second player is a distribution. 

For the first player, one can apply Proposition~\ref{prop:conv_function} to say that such a convex combination of $K_\epsilon$ functions can be expressed with a larger network that has $K_\epsilon \cdot p_\epsilon$ parameters in $[-R_\epsilon,R_\epsilon]$.

For the second player, once can apply Proposition~\ref{prop:conv_distrib} to get that  that such a uniform convex combination of $K_\epsilon$ functions can be expressed up to precision $1/K_\epsilon R_\epsilon$ with a larger network that has $K_\epsilon \cdot p_\epsilon$ parameters in $[-K_\epsilon R_\epsilon,K_\epsilon R_\epsilon]$.

Thus we get that a sufficient condition for $\epsilon$-approximate equilibrium of the game $\varphi$ to be achieved by a ReLU network with $p$ parameters in $[-R,R]$ is that,
\begin{equation}
   p \geq (p_\epsilon +6)K_\epsilon 
    \quad \text{and} \quad 
    R \geq K_\epsilon R_\epsilon
\end{equation}
Let us set,
\begin{equation}
    p_\epsilon := \left \lfloor \frac{\epsilon}{2D}\sqrt{\frac{p}{\ln(4 L R\sqrt{p}/\epsilon)}} -6\right \rfloor
    \quad \text{and} \quad 
    R_\epsilon := R \frac{p_\epsilon}{p} 
\end{equation}
 Using the fact that in Theorem~\ref{thm:finiteKepsilon},  $K_\epsilon \leq \frac{4 D^2}{\epsilon^2} p_\epsilon\log \Big( \frac{4 L R_\epsilon \sqrt{p_\epsilon}}{\epsilon} \Big)$ we have that 
 \begin{equation}
     (p_\epsilon + 6 ) K_\epsilon \leq \frac{\epsilon^2}{4D^2}\frac{p}{\ln(4 L R\sqrt{p}/\epsilon)}\frac{4 D^2}{\epsilon^2}\log \Big( \frac{4 L R \sqrt{p_\epsilon}}{\epsilon} \Big) \leq p 
 \end{equation}
 where we used the fact that $p_\epsilon \leq p$ and $R_\epsilon \leq R$. Moreover, since $p \geq (p_\epsilon +6)K_\epsilon$ we have that, 
 \begin{equation}
      K_\epsilon R_\epsilon \leq \frac{p}{p_\epsilon + 6} R_\epsilon \leq R\,.
 \end{equation}
Finally, since in Proposition~\ref{prop:conv_distrib} we approximate such uniform convex combination up to a TV distance $1/R$ and since we assumed that $\varphi$ was $\tilde L$-Lipschitz (with respect to the TV distance) we have the additional $\frac{\tilde L}{R}$ term. 
\end{proof}

\newpage


\end{document}